\documentclass{article}

\usepackage{arxiv}

\usepackage{epstopdf}
\usepackage{amsthm}
\usepackage{hyperref}
\usepackage{url}
\usepackage{hyperref}
\usepackage{url}
\usepackage[utf8]{inputenc} 
\usepackage[T1]{fontenc}    
\usepackage{booktabs}       
\usepackage{amsfonts}       
\usepackage{nicefrac}       
\usepackage{microtype}      
\usepackage{mathtools}
\usepackage{amssymb}
\usepackage[shortlabels]{enumitem}
\usepackage{xcolor,colortbl}
\usepackage{comment}
\usepackage{subcaption}
\usepackage{graphicx}
\usepackage{makecell}
\usepackage{multirow}
\usepackage{float}
\usepackage{booktabs,caption}
\usepackage[flushleft]{threeparttable}
\usepackage{ulem}
\usepackage{algorithm}
\usepackage{algpseudocode}
\usepackage[T1]{fontenc}
\usepackage{lmodern}

\newtheorem{theorem}{Theorem}
\newtheorem{definition}[theorem]{Definition}
\newtheorem{remark}[theorem]{Remark}
\newtheorem{lemma}[theorem]{Lemma}
\newtheorem{corollary}[theorem]{Corollary}
\newtheorem{assumption}[theorem]{Assumption}
\numberwithin{theorem}{section}

\title{Concentration Inequalities for Stochastic Optimization of Unbounded Objective Functions with Application to Denoising Score Matching}

\author{Jeremiah Birrell\\
  Department of Mathematics\\
  Texas State University\\
  San Marcos, TX,  USA \\
  \texttt{jbirrell@txstate.edu} 
}

\begin{document}

\maketitle

\begin{abstract}
 We derive novel concentration inequalities that bound the  statistical error  for a large class of stochastic optimization problems, focusing on the case of unbounded objective functions. Our derivations utilize the following key tools: 1)  A new  form of McDiarmid's inequality  that is based on sample-dependent  one-component mean-difference bounds and which leads to a novel uniform law of large numbers result for unbounded functions. 2) A new Rademacher complexity bound for families of functions that satisfy an appropriate sample-dependent Lipschitz property, which allows for application to a large class of distributions with unbounded support. As an application of these results, we derive statistical error bounds for denoising score matching (DSM), an application that inherently requires one to consider unbounded objective functions and distributions with unbounded support, even in cases where the data distribution has bounded support.  { In addition, our results quantify the benefit of sample-reuse in algorithms that employ easily-sampled auxiliary random variables in addition to the training data, e.g., as in DSM, which uses auxiliary Gaussian random variables. }
\end{abstract}


\section{Introduction}
In this work we derive a novel generalization of McDiarmid’s inequality along with a corresponding uniform law of large numbers (ULLN) result for unbounded functions and use these tools to obtain statistical error bounds  for stochastic optimization of unbounded objectives.   This work builds on the original bounded difference inequality  \cite{mcdiarmid1989method} as well the extensions in  \cite{kontorovich2014concentration,maurer2021concentration}, which apply to unbounded functions of a more  restrictive type than considered here.  More specifically, we obtain  concentration inequalities  for
\begin{align}\label{eq:ULLN_intro}
\sup_{\theta\in\Theta}\left\{\pm\left(\frac{1}{nm}\sum_{i=1}^n\sum_{j=1}^mg_\theta(x_i,y_{i,j})-E_{P_X\times P_Y}[g_\theta]\right)\right\}\,,
\end{align}
 as well as for the optimization error
\begin{align}\label{eq:stoc_opt_intro}
E_{P_X\times P_Y}[g_{\theta^*_{n,m}}]-\inf_{\theta\in\Theta} E_{P_X\times P_Y}\left[g_\theta\right]\,,
\end{align}
where $g_\theta:\mathcal{X}\times \mathcal{Y}\to\mathbb{R}$ is an appropriate  (unbounded) objective function, depending on parameters  $\theta\in\Theta$,  that satisfies certain sample-dependent (locally) Lipschitz conditions and $\theta^*_{n,m}$ is (approximately) a solution to the empirical optimization problem
\begin{align}\label{eq:empirical_opt_intro}
\inf_{\theta\in\Theta}\frac{1}{nm}\sum_{i=1}^n \sum_{j=1}^m g_{\theta}(X_i,Y_{i,j})\,,
\end{align}
with $X_i$ and $Y_{i,j}$ being independent samples from $P_X$ and $P_Y$ respectively.  We specifically focus on problems where there are naturally two types of samples, $X_i$ and $Y_{i,j}$,  with $m$-times more $Y$ samples than $X$ samples.  Such unbalanced sample numbers naturally occur in applications such as denoising score matching (DSM) and { generative adversarial networks (GANs)} which utilize easily sampled auxiliary random variables ($Y_{i,j}$) alongside the training data ($X_i$). We emphasize that our results are still meaningful and novel in the balanced ($m=1$) case and when there is only one type of sample due to their ability to handle a larger class of unbounded objectives and tail behaviors.

\subsection{Contributions}
The key contributions of the present work are as follows:
\begin{itemize}
\item  In Section \ref{sec:McDiarmid} we  derive  a   generalization of McDiarmid's inequality that can be applied to  families of unbounded functions and to  distributions with unbounded support. Our result utilizes non-constant bounds on one-component mean-differences  which, in particular, covers  cases of locally Lipschitz functions. In addition, we allow for a wide range of tail behaviors, beyond the sub-Gaussian and sub-exponential cases.
\item In Theorem \ref{thm:Rademacher_bound_unbounded_support_main} we  derive a  distribution-dependent Rademacher complexity bound which can be applied to   families of  unbounded functions that satisfy an appropriate sample-dependent Lipschitz condition.
\item We use the above two tools to derive uniform law of large numbers concentration inequalities in  Theorem \ref{thm:ULLN_var_reuse_mean_bound}  as well as  concentration inequalities for stochastic optimization problems in Theorem \ref{thm:conc_ineq_stoch_opt}, both of which allow for unbounded (objective) functions that satisfy appropriate sample-dependent (locally) Lipschitz conditions.  
\item { In Theorem \ref{thm:opt_L1_bound} we obtain an $L^1$ error bound on the stochastic optimization error    under even weaker assumptions; in particular, it applies to appropriate heavy-tailed $P_X$ and $P_Y$.} 
\item  We show that our results quantitatively capture the benefits of sample-reuse (see Eq.~\ref{eq:empirical_opt_intro}) in algorithms that  pair training data with  easily-sampled auxiliary random variables, e.g., Gaussian random variables in DSM. { As   applications of these tools, we study DSM    in Section \ref{sec:DSM} and  GANs in Section \ref{sec:GANs}}.
\end{itemize}

\subsection{Related Works}
Our generalization  of McDiarmid's inequality from the original bounded difference case in \cite{mcdiarmid1989method} to a much larger class of bounding functions can be viewed as an extension  of the result in \cite{kontorovich2014concentration}, which studied Lipschitz functions on a space with finite sub-Gaussian-diameter.  More recently,  \cite{maurer2021concentration}  took a different approach and extended  McDiarmid's inequality to sub-Gaussian and sub-exponential cases by using  entropy methods.  Our results, which apply to a wider range of tail behaviors, are also distinguished by the minimal assumptions they make regarding the one-parameter-difference bounding functions.  In particular, our framework applies to various forms of locally-Lipschitz behavior that naturally arise in a number of applications and which are not covered by   the methods of \cite{kontorovich2014concentration} or \cite{maurer2021concentration}.   This is especially important in our study of sample reuse, where we employ bounding functions on $X$-differences that are not uniform in the $Y$'s.  Our work shares some similarities with  the error analysis of  GANs \cite{liang2021well}, including the works  \cite{huang2022error,biau2021some} which consider distributions with unbounded support. However, our methods, { which we do apply to GANs in Section \ref{sec:GANs}},  allow for one to work directly with   unbounded functions, rather than using truncation arguments as in \cite{huang2022error},  or being restricted to the  sub-Gaussian  case as in \cite{biau2021some}.  We also note that the effect of sample reuse in  GANs, i.e., unbalanced number of samples from the data source and generator, is a simple special case of the  general  form of sample reuse considered here. Thus the GAN application only partially utilizes the novel features of our results.

An important  application that  uses essentially all of the novel features of our theorems is denoising score matching, which we study  in Section \ref{sec:DSM}. Since their introduction in \cite{songscore}, there has been great interest in the analysis of score-based generative models that use stochastic differential equation (SDE) (de)noising dynamics, and in  DSM in particular.   The works 
\cite{chen2023improved,pedrotti2023improved,bentonnearly,li2024towards,li2024sharp,potaptchik2024linear,chen2024equivariant,mimikos2024score} all derive convergence results or error bounds for score-based generative models under various assumptions on the score-matching error.  Closest to our work on DSM is \cite{oko2023diffusion}, which  obtained $L^1$ error bounds as well as bounds on the estimation error in the form of expected total variation and $1$-Wasserstein distance, in contrast to the  concentration inequalities we obtain    below.  The key  inputs to the   bounds in  \cite{oko2023diffusion}, such as the covering number bound in their Lemma 4.2, assume a uniformly bounded score-network hypothesis class; see their definition of $\mathcal{S}$ above their Lemma 4.2.  While the authors  prove the  necessary approximation results  to justify and account for the errors made in the restriction to score  models, a benefit of our framework is that one can work directly with  unbounded objective functions, including an unbounded score model; the goal of our study of DSM in  Section \ref{sec:DSM} is to demonstrate this strength via derivations of $L^1$ error bounds and a novel concentration inequality.  These strengths carry over to other stochastic optimization problems that are naturally posed on an unbounded domain and with an unbounded objective function.

\section{A Generalization of McDiarmid's Inequality for Unbounded Functions}\label{sec:McDiarmid}

We begin our analysis by deriving a new form of McDiarmid’s inequality that utilizes   sample-dependent  bounds on one-component mean-differences.  Our result   builds on the earlier generalization from \cite{kontorovich2014concentration} which applied to unbounded Lipschitz functions on spaces of finite sub-Gaussian diameter.  Here we make several innovations: 1) We drop the Lipschitz requirement, allowing for more general bounding functions   on  one-component  differences.  2)  Our method allows bounding functions  to have unbounded dependence on all later variables in the ordering;  this is necessary in order to handle the  locally-Lipschitz behavior of  many natural objective functions.  3) We consider   substantially more general tail-behaviors, generalizing the sub-Gaussian case studied in \cite{kontorovich2014concentration}. All three of these generalizations  will be important for our application to DSM in Section \ref{sec:DSM}.   As previously noted, \cite{maurer2021concentration} took a different approach to generalizing McDiarmid's inequality and studied both the sub-Gaussian and sub-exponential cases; however, their results do not allow for the kind of sample-dependent Lipschitz constants that are crucial for our subsequent applications.   We   emphasize that, while the proof strategy we employ below follows the pattern pioneered by \cite{kontorovich2014concentration}, the assumptions made here are substantially weaker, thus allowing for  a much wider range of applications than previously possible. Therefore we expect the following result, both the moment generating function (MGF) bound and concentration inequality, to be of independent interest, beyond our uses of it in the following sections.  

\begin{theorem}\label{thm:McDiarmid_general_mean_bound_method}
Let $(\mathcal{X}_i,P_i)$, $i=1,...,n$ be probability spaces, define $\mathcal{X}^n\coloneqq\prod_{i=1}^n \mathcal{X}_i$ with the product $\sigma$-algebra, and let $P^n\coloneqq\prod_{i=1}^n P_i$.  Suppose   $\phi:\mathcal{X}^n\to\mathbb{R}$ satisfies the following:
\begin{enumerate}
\item Integrability: $\phi\in L^1(P^n)$ and  $\phi(x_1,...,x_i,\cdot)\in L^1(\prod_{j>i} P_j)$ for all $i\in\{1,...,n-1\}$ and all $(x_1,...,x_i)\in\prod_{j=1}^i \mathcal{X}_j$.
\item One-component mean-difference bounds: For all  $i\in\{1,...,n\}$ we have measurable $h_i:\mathcal{X}_i\times\mathcal{X}_i\to[0,\infty]$ such that  
\begin{align}\label{eq:mean_diff_bound_assump}
\left|E_{\prod_{j>i} P_j}[\phi(x_1,...,x_{i-1},x_i,\cdot)]-E_{\prod_{j>i} P_j}[\phi(x_1,...,x_{i-1},\tilde{x}_i,\cdot)]\right|\leq h_i(x_i,\tilde x_i)
\end{align}
for all   $(x_1,...,x_i)\in\prod_{j=1}^i \mathcal{X}_j$, $\tilde{x}_i\in \mathcal{X}_i$. Note that for $i=n$ this assumption should be interpreted as requiring
\begin{align}\label{eq:mean_diff_bound_assump_n}
|\phi(x_1,...,x_{n-1},x_n)-\phi(x_1,...,x_{n-1},\tilde{x}_n)|\leq h_n(x_n,\tilde x_n)\,.
\end{align}
\item MGF bounds: We have    $\xi_i:[0,\infty)\to [0,\infty]$, $i=1,...,n$, such that 
\begin{align}\label{eq:cosh_MGF_bound}
E_{P_i\times P_i}[\cosh(\lambda h_i)]\leq e^{\xi_i(|\lambda|)} \,\,\,\text{ for all $\lambda\in \mathbb{R}$\,.}
\end{align}
\end{enumerate}
  Then  we have the MGF bound
\begin{align}\label{eq:new_McDiarmid_MGF}
E_{P^n}\!\left[e^{\lambda (\phi-E_{P^n}[\phi])}\right]\leq \exp\left(\sum_{i=1}^n \xi_i(|\lambda|)\right) \,\,\, \text{  for all $\lambda\in \mathbb{R}$}
\end{align}
and for all $t\geq 0$ we have the concentration inequality
\begin{align}\label{eq:new_McDiarmid_conc}
P^n(\pm(\phi-E_{P^n}[\phi])\geq t)\leq \exp\left(-\sup_{\lambda\in[0,\infty)}\left\{\lambda t-\sum_{i=1}^n \xi_i(\lambda)\right\}\right)\,.
\end{align}

\end{theorem}
\begin{remark}\label{remark:cosh_dzeta}
In  equation \eqref{eq:cosh_MGF_bound} and below one should interpret $0\cdot \infty\coloneqq0$, $\cosh(\pm\infty)\coloneqq\infty$, and $\exp(\infty)\coloneqq \infty$.   We allow the $\xi_i$'s to take the value infinity for simplicity of the statement of the theorem, though of course this result is only meaningful when the $\xi_i$'s are all finite on some neighborhood of $0$; if they are all finite on $[0,K)$ for some $K>0$ then the supremum on the right-hand side of \eqref{eq:new_McDiarmid_conc} can be restricted to $\lambda\in [0,K)$.  Note that equation \eqref{eq:cosh_MGF_bound} is equivalent to having a MGF bound for $\zeta h_i(x_i,\tilde x_i)$ under the distribution $d\zeta\times P_i(dx_i)\times P_i(d\tilde{x}_i)$, where $d\zeta$ is the uniform distribution on $\{-1,1\}$. See, e.g., Chapter 2 in \cite{vershynin2018high} or in \cite{wainwright2019high} for techniques that can be used to derive such MGF bounds.  Finally,  we note that the bounding functions $h_i$ in \eqref{eq:mean_diff_bound_assump} can be allowed to depend on the previous variables $x_1,...,x_{i-1}$ as long as   \eqref{eq:cosh_MGF_bound} then holds uniformly   in those previous variables.   We do not have a need for that more general version   in this paper so we present it in Appendix \ref{app:McDiarmid_general_mean_bound_method2}.
\end{remark}
\begin{proof}
To obtain the claimed MGF bound \eqref{eq:new_McDiarmid_MGF} we  follow the strategy of  \cite{kontorovich2014concentration}, while noting that the assumptions can be substantially weakened to  cover the much more general  one-component mean-difference  property \eqref{eq:mean_diff_bound_assump}. Start by defining $\phi_n\coloneqq \phi$,
\begin{align}
\phi_i(x_1,...,x_i)\coloneqq E_{\prod_{j>i}P_j}[\phi(x_1,...,x_i,\cdot) ]
\end{align}
for $i=1,...,n-1$, $\phi_0\coloneqq E_{P^n}[\phi]$, and $V_i\coloneqq\phi_i-\phi_{i-1}$, $i=1,...,n$, so that we have the telescoping sum
\begin{align}
\sum_{i=1}^n V_i=\phi-E_{P^n}[\phi]\,.  
\end{align}

For  $\lambda\in \mathbb{R}$ and $i\in\{1,...,n\}$ we can bound the MGF of  $V_i$ in its last variable as follows. 
\begin{align}\label{eq:gen_McDiarmid_derivation}
&\int e^{\lambda V_i(x_1,...,x_i)} P_i(dx_i)\\
=&\int e^{\lambda   \int \left(E_{\prod_{j>i}P_j}[\phi(x_1,...,x_{i-1},x,\cdot) ]- E_{\prod_{j>i}P_j}[\phi(x_1,...,x_{i-1},y,\cdot) ]\right)P_i(dy)} P_i(dx)\notag\\
\leq&\int  \int e^{\lambda  (E_{\prod_{j>i}P_j}[\phi(x_1,...,x_{i-1},x,\cdot) ]- E_{\prod_{j>i}P_j}[\phi(x_1,...,x_{i-1},y,\cdot) ])}P_i(dy) P_i(dx)\notag\\
=&\int \cosh\left(\lambda  \left(E_{\prod_{j>i}P_j}[\phi(x_1,...,x_{i-1},x,\cdot) ]- E_{\prod_{j>i}P_j}[\phi(x_1,...,x_{i-1},y,\cdot) ]\right)\right)(P_i\times P_i)(dxdy) \notag\\
\leq&  \int \cosh( \lambda h_i(x,y)) (P_i\times P_i)(dxdy)\leq  e^{\xi_i(|\lambda|)}\,.\notag
\end{align}
To obtain the third line we used Jensen's inequality. The fourth line follows from symmetrization.  The first inequality on the fifth line follows from the symmetry and monotonicity properties of $\cosh(t)$ together with the one-component mean-difference bound \eqref{eq:mean_diff_bound_assump} while the last inequality follows from the assumption \eqref{eq:cosh_MGF_bound}.

Therefore, for  $\lambda\in  \mathbb{R}$ one obtains the following centered MGF bound for $\phi$ by induction:
\begin{align}
\int e^{\lambda(\phi-E_{P^n}[\phi])} dP^n=\int\prod_{i=1}^n  e^{\lambda V_i} dP^n\leq \int \prod_{i=1}^k e^{\lambda V_i}dP^k\prod_{i=k+1}^n e^{\xi_i(|\lambda|)}\leq \exp\left(\sum_{i=1}^n \xi_i(|\lambda|)\right)\,.
\end{align}
From this, \eqref{eq:new_McDiarmid_conc} then follows via a standard Chernoff bound, see, e.g., Section 2.1 in \cite{boucheron2013concentration}.
\end{proof}

If one has one-component difference bounds
\begin{align}\label{eq:one_comp_diff_tildeh_i}
|\phi(x_1,...,x_n)-\phi(x_1,...,x_{i-1},\tilde{x}_i,x_{i+1},...x_n)|\leq \tilde{h}_i(x_i,\tilde{x}_i,x_{i+1},...,x_n)
\end{align}
for some  measurable $\tilde{h}_i$'s (which are allowed to depend on all later variables)  then the following points  can be useful when checking the various conditions required by Theorem \ref{thm:McDiarmid_general_mean_bound_method}.
\begin{enumerate}
\item  In cases where \eqref{eq:one_comp_diff_tildeh_i} holds, the required integrability conditions on $\phi$ follow if the $\tilde{h}_i$ are real-valued and satisfy $\tilde{h}_i\in L^1(P_i\times P_i\times\prod_{j>i}P_j)$ for $i\in\{1,..,n\}$. This can be verified as follows.
Fubini's theorem applied to $\tilde{h}_i$ implies $D_i\coloneqq\{({y}_i,y_{i+1},...,y_n): \int \tilde{h}_i(\tilde{x}_i,{y}_i,y_{i+1},...,y_n)P_i(d\tilde{x}_i) <\infty\}$ satisfies $(\prod_{j\geq i}P_j)(D_i)=1$. Therefore $\widetilde{D}_i\coloneqq\left(\prod_{\ell=1}^{i-1} \mathcal{X}_\ell\right)\times D_i$ satisfies $P^n(\widetilde{D}_i)=1$ for all $i$, and hence  there exists $(y_1,...,y_n)\in \cap_{i=1}^n \widetilde{D}_i$.  Therefore $({y}_i,y_{i+1},...,y_n)\in D_i$ for all $i$, i.e.,
\begin{align}
 \int \tilde{h}_i(\tilde{x}_i,{y}_i,y_{i+1},...,y_n) P_i(d\tilde{x}_i) <\infty
\end{align}
for all $i$.
For all $x\in\mathcal{X}^n$ we then have
\begin{align}
|\phi(x)|\leq&|\phi(y)|+ \sum_{j=1}^n|\phi(x_1,...,x_j,y_{j+1},...,y_n)-\phi(x_1,...,x_{j-1},y_{j},...,y_n)|\\
\leq&|\phi(y)|+ \sum_{j=1}^n \tilde{h}_j(x_j,y_j,y_{j+1},...,y_n)\,,\notag
\end{align}
and hence for $i=0,...,n-1$ we can compute
\begin{align}
&\int |\phi(x_1,...,x_n)|\prod_{\ell=i+1}^n P_\ell(dx_\ell)\leq  |\phi(y)|+ \sum_{j=1}^n\int \tilde{h}_j(x_j,y_j,y_{j+1},...,y_n)\prod_{\ell=i+1}^n P_\ell(dx_\ell)\\
=& |\phi(y)|+ \sum_{j=1}^i \tilde{h}_j(x_j,y_j,y_{j+1},...,y_n)+ \sum_{j=i+1}^n\int \tilde{h}_j(x_j,y_j,y_{j+1},...,y_n) P_j(dx_j)<\infty\,.\notag
\end{align}
Thus, for  $i=1,...,n-1$ we conclude that $\phi(x_1,...,x_i,\cdot)\in L^1(\prod_{\ell>i} P_\ell)$ for all $x_1,...,x_i$  and  also that $\phi\in L^1(P^n)$.
\item If  \eqref{eq:one_comp_diff_tildeh_i} holds then \eqref{eq:mean_diff_bound_assump}-\eqref{eq:mean_diff_bound_assump_n} hold if one takes $h_i(x_i,\tilde{x}_i)=E_{\prod_{j>i} P_j}\!\left[\tilde{h}_i(x_i,\tilde{x}_i,\cdot)\right]$ for $i<n$ and $h_n=\tilde{h}_n$.   
\item  In cases where \eqref{eq:one_comp_diff_tildeh_i} holds, to obtain the MGF bound \eqref{eq:cosh_MGF_bound} one possible strategy is to first employ Jensen's inequality to  bound
\begin{align}
\cosh\left(\lambda E_{\prod_{j>i} P_j}[\tilde{h}_i(x_i,\tilde{x}_i,\cdot)]\right)\leq& E_{\prod_{j>i} P_j}[\cosh(\lambda\tilde{h}_i(x_i,\tilde{x}_i,\cdot))]\,.
\end{align}
Hence \eqref{eq:cosh_MGF_bound} is implied by a bound of the form
\begin{align}\label{eq:h_i_MFG_bound}
\int \cosh(\lambda \tilde{h}_i(x_i,\tilde{x}_i,x_{i+1},...,x_n))P_i(dx_i) P_i(d\tilde{x}_i)\prod_{j>i} P_j(dx_j)\leq e^{\xi_i(|\lambda|)}\,.
\end{align}
Similarly to Remark \ref{remark:cosh_dzeta}, the left-hand side of \eqref{eq:h_i_MFG_bound} can also be viewed as a MGF and hence can be bounded using a number of well-known techniques. 

\end{enumerate}

\subsection{Use of Orlicz Norms}\label{sec:Orlicz}
{
To apply Theorem \ref{thm:McDiarmid_general_mean_bound_method} one requires bounds on the expectation of $\cosh(\lambda X)$ for various random variables, $X$.  In the application to DSM in Section \ref{sec:DSM} we do this by using  properties of sub-exponential and sub-Gaussian random variables.  More generally, given a  (non-decreasing) Young function $\Psi$, if $X$ has finite Orlicz (Luxemburg) norm, $\|X\|_\Psi$ (if we want to emphasize the choice of probability measure $\mathbb{P}$ we use the notation $\|\cdot\|_{\Psi,\mathbb{P}}$),  one can  obtain   bounds   in terms of $\|X\|_{\Psi}$ as follows  (see, e.g., Chapter 2.2 in \cite{van2013weak} for background on Orlicz norms). For proofs of the results in this subsection, see Appendix \ref{app:additional_proofs}.
\begin{lemma}\label{lemma:cosh_bound_Orliz}
Let $\Psi$ be a non-decreasing Young function  and   $X$ be a real-valued random variable with $\|X\|_\Psi<\infty$. Then for all $\lambda\in\mathbb{R}$ we have
\begin{align}
\mathbb{E}[\cosh(\lambda X)]\leq&1+\sum_{k=1}^\infty  \frac{\lambda^{2k} }{(2k)!}  \int_0^\infty  \min\left\{\frac{1}{\Psi(t^{1/(2k)}/\|X\|_\Psi)},1\right\}dt\,.
\end{align}
\end{lemma}
Of particular interest are the Young functions $\Psi_q(t)\coloneqq \exp(t^q)-1$ for $q\in[1,\infty)$ (recall that $\|X\|_{\Psi_1}<\infty$ iff $X$ is sub-exponential and $\|X\|_{\Psi_2}<\infty$ iff $X$ is sub-Gaussian).  In this case, using the bound $\min\{a^{-1},1\}\leq 2/(a+1)$ for all $a>0$, one can derive the following   corollary.
\begin{corollary}\label{cor:cosh_bound_Orlicz_Psi_q}
Let $q\in[1,\infty)$ and    $X$ be a real-valued random variable with $\|X\|_{\Psi_q}<\infty$. Then for  all $\lambda\in\mathbb{R}$ we have
\begin{align}\label{eq:cosh_bound_Psi_q}
\mathbb{E}[\cosh(\lambda X)]\leq&  1+2\sum_{k=1}^\infty (\|X\|_{\Psi_q}\lambda)^{2k}  \frac{\Gamma(2k/q+1) }{(2k)!}    \,.
\end{align}
\end{corollary}
Note that the right-hand side of \eqref{eq:cosh_bound_Psi_q} can be written as $\exp(\xi(|\lambda|))$ for $\xi$ satisfying $\xi(|\lambda|)=O(\lambda^2)$ as $\lambda\to 0$.  When $q=1$ the series can be evaluated  to give 
\begin{align}\label{eq:mgf_bound_subexp}
\mathbb{E}[\cosh(\lambda X)]\leq&  1+\frac{2 (\|X\|_{\Psi_1}\lambda)^{2}}{1- (\|X\|_{\Psi_1}\lambda)^{2}}\leq\exp\left(\frac{2 (\|X\|_{\Psi_1}\lambda)^{2}}{1- (\|X\|_{\Psi_1}\lambda)^{2}}\right)\,,\,\,\,\,\,|\lambda|<1/ \|X\|_{\Psi_1}    \,.
\end{align}
In cases where $q>1$, \eqref{eq:cosh_bound_Psi_q} can be further bounded, e.g., using Stirling's formula (with error bound) for the gamma function. The result \eqref{eq:mgf_bound_subexp} is similar to the MGF bound implied by the Bernstein condition, see Eq.~(2.15)-(2.16)  in \cite{wainwright2019high}, though the denominator differs  due to the fact that only the even moments of $X$ contribute here. Combining  \eqref{eq:cosh_bound_Psi_q}-\eqref{eq:mgf_bound_subexp} with   Theorem \ref{thm:McDiarmid_general_mean_bound_method}, we obtain the following sub-Gaussian and  Bernstein-type concentration inequalities.
\begin{corollary}\label{cor:subexp_subGauss_result}
Let $(\mathcal{X}_i,P_i)$, $i=1,...,n$ be probability spaces and $\phi:\mathcal{X}^n\to\mathbb{R}$ be measurable. Suppose $\phi$ satisfies $\eqref{eq:one_comp_diff_tildeh_i}$ for real-valued $\tilde{h}_i\in L^1(P_i\times P_i\times\prod_{j>i}P_j)$  and  define $h_i(x_i,\tilde{x}_i)\coloneqq E_{\prod_{j>i} P_j}[\tilde{h}_i(x_i,\tilde{x}_i,\cdot)]$.
\begin{enumerate}
\item If the $h_i$ are sub-Gaussian under $P_i\times P_i$ for all $i$ then for all $t> 0$ we have
\begin{align}\label{eq:new_McDiarmid_conc_sub-Gauss}
P^n(\pm(\phi-E_{P^n}[\phi])\geq t)\leq& \exp\left(-\frac{t^2}{4\sum_{i=1}^n\|h_i\|_{\Psi_2,P_i^2}^{2}}\right)\,.
\end{align}
\item If the $h_i$ are   sub-exponential under $P_i\times P_i$ for all $i$ then for all $t> 0$ we have
\begin{align}\label{eq:new_McDiarmid_conc_sub-exp}
P^n(\pm(\phi-E_{P^n}[\phi])\geq t)\leq&  \exp\left(-\frac{t^2}{8 \sum_{i=1}^n\|h_i\|_{\Psi_1,P_i^2}^2 +2t \max_i\{\|h_i\|_{\Psi_1,P_i^2}\}}\right) \,.
\end{align}
\end{enumerate}
\end{corollary}

}

\subsection{Comparison of Theorem \ref{thm:McDiarmid_general_mean_bound_method} with McDiarmid's Inequality and Related Methods}
The generalization of McDiarmid's inequality in  \cite{kontorovich2014concentration}, which to the best of the author's knowledge is the closest extant result to Theorem \ref{thm:McDiarmid_general_mean_bound_method}, assumes that $\phi$ is Lipschitz in each component, i.e., that $\tilde{h}_i(x_i,\tilde{x}_i,x_{i+1},...,x_n)=d_i(x_i,\tilde{x}_i)$ in \eqref{eq:one_comp_diff_tildeh_i}, where $d_i$ is a metric on $\mathcal{X}_i$. In this case, our result reduces to theirs by taking $\xi_i(\lambda)=\Delta_{\text{SG}}(\mathcal{X}_i)^2\lambda^2/2$, where $\Delta_{\text{SG}}(\mathcal{X}_i)$ denotes the sub-Gaussian diameter of $(\mathcal{X}_i,P_i)$, as defined in   \cite{kontorovich2014concentration}.  However, this Lipschitz case does not cover the type of locally Lipschitz behavior that is inherent in applications such as DSM.

Regarding the more recent entropy-based approach in \cite{maurer2021concentration}, our Theorem \ref{thm:McDiarmid_general_mean_bound_method} produces similar results      in the case \eqref{eq:one_comp_diff_tildeh_i}, when $\tilde{h}_i$  is either sub-Gaussian or sub-exponential in $(x_i,\tilde{x}_i)$, uniformly in the remaining variables. Specifically, compare part 1 of Corollary \ref{cor:subexp_subGauss_result} above with their Theorem 3 and compare part 2  with their  Theorem 4 and especially their Theorem 11, but note the definitions of  $\|\cdot\|_{\psi_\alpha}$ in  \cite{maurer2021concentration}  are related to, but not  the same as, our Orlicz norms.  Despite the similarities in these cases, the results in \cite{maurer2021concentration} are distinct from ours in general.  The most significant advantage of  our approach is that Theorem \ref{thm:McDiarmid_general_mean_bound_method} allows for non-uniform behavior in the subsequent variables. This  is crucial in applications involving locally-Lipschitz functions/losses, such as DSM; in such cases the terms $\|f_k(X)\|_{\psi_\alpha}$ from  \cite{maurer2021concentration} are unbounded and hence    the upper bounds in  Theorems 3-5   of \cite{maurer2021concentration} become uninformative.

The classical McDiarmid inequality, i.e., bounded differences inequality, from \cite{mcdiarmid1989method}  corresponds to the $h_i$ in \eqref{eq:mean_diff_bound_assump} being constants (more specifically, to the $\tilde{h}_i$ in \eqref{eq:one_comp_diff_tildeh_i} being constants).  As was previously observed in \cite{kontorovich2014concentration}, the tools we have used in the proof of Theorem \ref{thm:McDiarmid_general_mean_bound_method}   lead to a worse constant in the exponent than is obtained in the classical bounded differences case when both methods are applicable. { To see this, note that when \eqref{eq:one_comp_diff_tildeh_i} holds for constant $\tilde{h}_i=c_i\in [0,\infty)$ then McDiarmid's inequality corresponds to  \eqref{eq:new_McDiarmid_MGF}    with $\xi_i(\lambda)=\lambda^2c_i^2/8$.  However, in this case the  tightest possible bound of the form \eqref{eq:cosh_MGF_bound} with exponent quadratic in $\lambda$  is $E_{P_i\times P_i}[\cosh(\lambda h_i)]=\cosh(\lambda c_i)\leq e^{\lambda^2c_i^2/2}$. Thus the constant in the exponent is worse that McDiarmid's inequality by a factor of $4$. The worse constant stems from the use of Jensen's inequality and symmetrization in the string of computations \eqref{eq:gen_McDiarmid_derivation}, as opposed to the use of Hoeffding's lemma when deriving McDiarmid's inequality.  Here, the aforementioned techniques were crucial for dealing with the possibly unbounded  one-component differences, which cannot be handled via Hoeffding's lemma.} Thus, while McDiarmid's inequality is preferable when utilizing uniform bounds on one-component differences, Theorem \ref{thm:McDiarmid_general_mean_bound_method} significantly extends the class of $\phi$'s and probability distributions for which one can obtain concentration inequalities of the form \eqref{eq:new_McDiarmid_conc} and so allows for novel applications, such as those considered in Section \ref{sec:applications} below.  A summary of the above comparisons can be found in Table \ref{table:comparison}.

\begin{table}[h!]
\centering\small
\begin{tabular}{||l| c c c c||} 
 \hline
 & Theorem \ref{thm:McDiarmid_general_mean_bound_method} &  Entropy   & Sub-Gaussian  & McDiarmid's  \\
 & & method  \cite{maurer2021concentration}& diameter \cite{kontorovich2014concentration}   &  inequality  \cite{mcdiarmid1989method}\\
 [0.5ex] 
 \hline\hline
Tail behaviors & General MGF or & sub-Gaussian or & sub-Gaussian&compact \\
covered & Orlicz-norm bound &sub-exponential  &  & support\\
\hline
Allows non-uniform &  &  &  &\\
dependence on   & yes & no & no &no\\
later variables  &  &  &  &\\
 \hline
Tightest constants   & no & no & no &yes\\
 \hline
\end{tabular}
\vspace{0.5mm}
\caption{Comparison of our Theorem \ref{thm:McDiarmid_general_mean_bound_method} with  the entropy-based method of \cite{maurer2021concentration}, the sub-Gaussian diameter approach of \cite{kontorovich2014concentration}, and McDiarmid's inequality  \cite{mcdiarmid1989method}. In cases where multiple methods apply, the sub-Gaussian diameter approach \cite{kontorovich2014concentration} is a sub-case of our Theorem \ref{thm:McDiarmid_general_mean_bound_method} but neither  Theorem \ref{thm:McDiarmid_general_mean_bound_method} nor the entropy method from \cite{maurer2021concentration} subsume one another. Of these methods,  McDiarmid's inequality produces the tightest constants when utilizing uniform bounds on one-component differences.}\label{table:comparison}
\end{table}

\section{A Uniform Law of Large Numbers Concentration Inequality for Unbounded Functions with Sample Reuse}\label{sec:ULLN}
In this section we use Theorem \ref{thm:McDiarmid_general_mean_bound_method} to obtain a novel  ULLN concentration inequality  for families of unbounded functions and which allows for sample reuse. More specifically, we assume one has two types of samples, $X$ with values in $\mathcal{X}$ and $Y$ with values in $\mathcal{Y}$, and that one has $m$-times more $Y$ samples than $X$ samples.  This situation often occurs in applications, with the $X$'s being the training data and the $Y$'s being samples from some auxiliary  distribution that are used in the training algorithm and which can easily be generated at will; in such cases, $m$ can  often be thought of as the number of passes through the data. Our result shows that taking the reuse multiple $m\to \infty$ results in  bounds where the  $Y$-dependence is  integrated out, corresponding to the intuition that one can replace the average over $Y$'s with an expectation.  This property allows us to obtain a ULLN result that is uniform over all $P_X$'s that are supported in a given compact set,  even when the (known) distribution of the $Y$'s has unbounded support; such capability is useful for analyzing methods such as DSM { and GANs} that incorporate such unbounded $Y$'s.   To derive the general result, we work under the following assumptions. Later we will show how they can  be checked  in  applications.
\begin{assumption}\label{assump:ULLN}
Let $(\mathcal{X},P_X)$, $(\mathcal{Y},P_Y)$ be probability spaces and suppose the following:  
\begin{enumerate}
\item We have  a nonempty  collection $\mathcal{G}$ of measurable functions $g:\mathcal{X}\times\mathcal{Y}\to\mathbb{R}$.
\item We have  $h_{\mathcal{X}}\in L^1(P_X\times P_X\times P_Y)$ such that $|g(x,y)-g(\tilde{x},y)|\leq h_{\mathcal{X}}(x,\tilde{x},y)<\infty$ for all $x,\tilde x\in \mathcal{X}$, $y\in \mathcal{Y}$, $g\in\mathcal{G}$.
\item We have $h_{\mathcal{Y}}\in L^1(P_Y\times P_Y)$ such that $|g(x,y)-g(x,\tilde{y})|\leq h_{\mathcal{Y}}(y,\tilde{y})<\infty$ for all $x\in \mathcal{X}$, $y,\tilde{y}\in \mathcal{Y}$, $g\in\mathcal{G}$.
\item There exists $K_X\in(0,\infty]$, $\xi_X:[0,K_X)\to[0,\infty)$ such that
\begin{align}\label{eq:xi_X_def}
\int  \cosh\left(\lambda  E_{P_Y}[h_{\mathcal{X}}(x,\tilde{x},\cdot)]\right)  (P_X\times P_X)(dxd\tilde{x}) \leq e^{\xi_X(|\lambda|)}\,\,\,\,\text{ for all $\lambda\in(-K_X,K_X)$.}
\end{align}
\item There exists $K_Y\in(0,\infty]$, $\xi_Y:[0,K_Y)\to[0,\infty)$ such that
\begin{align}\label{eq:xi_Y_def}
\int \cosh\left(\lambda h_{\mathcal{Y}}(y,\tilde{y})\right)  (P_Y\times P_Y)(dyd\tilde{y})\leq e^{\xi_Y(|\lambda|)}\,\,\,\,\text{ for all $\lambda\in(-K_Y,K_Y)$.}
\end{align}
 \item  We have  $m\in\mathbb{Z}^+$ such that $K_X/m\leq K_{Y}$ .\label{point:sample_ratio}
\end{enumerate}
\begin{remark}
Equations \eqref{eq:xi_X_def} and  \eqref{eq:xi_Y_def} are equivalent to centered MGF bounds; see also Remark \ref{remark:cosh_dzeta}. { Consistent with our previous notation, the parameter $m$ in point \ref{point:sample_ratio} above refers to the sample reuse ratio, i.e., the number of $Y$-samples per $X$-sample.}
\end{remark}
\end{assumption}
Under the above assumptions, we can use Theorem \ref{thm:McDiarmid_general_mean_bound_method} to obtain the following. 
\begin{lemma}\label{lemma:ULLN_var_reuse}
Let $\mathcal{G}$ be a  countable collection of measurable functions $g:\mathcal{X}\times\mathcal{Y}\to\mathbb{R}$ for which
 Assumption \ref{assump:ULLN} holds. For $n\in\mathbb{Z}^+$ define $P^{n,m}\coloneqq \prod_{i=1}^n \left(P_X\times \prod_{j=1}^m P_Y\right)$ and  define $\phi: \prod_{i=1}^n\left(\mathcal{X}\times \mathcal{Y}^m\right)\to  \mathbb{R}$ by
\begin{align}\label{eq:ULLN_phi_def}
&\phi(x_1,y_{1,1},...,y_{1,m},....,x_n,y_{n,1},...,y_{n,m})\coloneqq\sup_{g\in\mathcal{G}}\left\{\frac{1}{nm}\sum_{i=1}^n\sum_{j=1}^mg(x_i,y_{i,j})-E_{P_X\times P_Y}[g]\right\}\,.
\end{align}
Then $\phi \in L^1({P^{n,m}})$ and  for all $t\geq 0$ we have the concentration inequality
\begin{align}\label{eq:new_ULLN_conc}
P^{n,m}(\pm(\phi-E_{P^{n,m}}[\phi])\geq t)
\leq& \exp\left(-n\sup_{\lambda\in[0,K_X)}\left\{\lambda t-\left( \xi_X(\lambda)+m \xi_Y(\lambda/m)\right)\right\}\right)\,.
\end{align}
\end{lemma}
\begin{remark}
Note that when  $\xi_{Y}(\lambda)$ is $o(\lambda)$ as $\lambda\to 0$, e.g., in the   sub-Gaussian and sub-exponential cases, then for large  $m$  the supremum on the right-hand side of \eqref{eq:new_ULLN_conc} is approximately  equal to the Legendre transform of $\xi_X$. Therefore the contribution of the $Y$ samples  is approximated by simply integrating out the  $y$-dependence in the function bounding the $x$-differences, $h_{\mathcal{X}}$; see \eqref{eq:xi_X_def}.    We also note that our  concentration inequality provides new insight even in the case where there is no sample reuse (e.g., where $h_{\mathcal{X}}$ and the $g$'s do not depend on $\mathcal{Y}$ and $h_\mathcal{Y}\coloneqq 0$, $\xi_Y\coloneqq 0$) for two reasons: 1) It is able to utilize general MGF bounds.  2) The minimal assumptions made on the form of ${h}_\mathcal{X}$, which allows for locally-Lipschitz behavior that is not covered by other approaches.  In the case of unbounded functions, one often needs both of the features (1)  and (2), whether or not there is sample reuse.
\end{remark}
\begin{proof}  
We will check that all of the conditions required by Theorem \ref{thm:McDiarmid_general_mean_bound_method} are satisfied.
The assumption $h_{\mathcal{X}}\in L^1(P_X\times P_X\times P_Y)$ and  $h_{\mathcal{Y}}\in L^1(P_Y\times P_Y)$ imply that
 $\int h_{\mathcal{X}}(x,x^\prime,y^\prime) P_X(dx) <\infty$ for $P_X\times P_Y$-a.e. $({x}^\prime,y^\prime)$ and  $\int h_{\mathcal{Y}}(y,{y}^\prime)P_Y(dy)<\infty$ for $P_Y$-a.e. ${y}^\prime$.  Therefore there exists $x^\prime,y^\prime$ such that both $\int h_{\mathcal{X}}(x,x^\prime,y^\prime) P_X(dx) <\infty$ and   $\int h_{\mathcal{Y}}(y,{y}^\prime)P_Y(dy)<\infty$.

Let $g\in\mathcal{G}$.  For all $x,y$ we have
\begin{align}
|g(x,y)|\leq& |g(x,y)-g(x,y^\prime)|+|g(x,y^\prime)-g(x^\prime,y^\prime)|+|g(x^\prime,y^\prime)|\\
\leq& h_{\mathcal{Y}}(y,y^\prime)+h_{\mathcal{X}}(x,x^\prime,y^\prime)+|g(x^\prime,y^\prime)|\notag
\end{align}
and hence
\begin{align}
E_{P_{X}\times P_Y}[|g|]\leq  \int h_{\mathcal{Y}}(y,y^\prime)P_Y(dy)+\int h_{\mathcal{X}}(x,x^\prime,y^\prime)P_X(dx)+|g(x^\prime,y^\prime)|<\infty\,.
\end{align}
Therefore $\mathcal{G}\subset L^1(P_X\times P_Y)$. For $g\in\mathcal{G}$ we have
\begin{align}
&\left|\frac{1}{nm}\sum_{i=1}^n\sum_{j=1}^mg(x_i,y_{i,j})-E_{P_X\times P_Y}[g]\right|\\
\leq&\frac{1}{nm}\sum_{i=1}^n\sum_{j=1}^m|g(x_i,y_{i,j})-g(x_i,y^\prime)|+\frac{1}{n}\sum_{i=1}^n|g(x_i,y^\prime)-g(x^\prime,y^\prime)|\notag\\
&+\int |g(x^\prime,y^\prime)-g(x,y^\prime)|P_X(dx)+\int\int |g(x,y^\prime)-g(x,y)|P_X(dx)P_Y(dy)\notag\\
\leq&\frac{1}{nm}\sum_{i=1}^n\sum_{j=1}^m h_{\mathcal{Y}}(y_{i,j},y^\prime)+\frac{1}{n}\sum_{i=1}^n h_{\mathcal{X}}(x_i,x^\prime,y^\prime)\notag\\
&+\int h_{\mathcal{X}}(x,x^\prime,y^\prime)P_X(dx)+\int h_{\mathcal{Y}}(y,y^\prime)P_Y(dy)<\infty\,.\notag
\end{align}
Therefore  $\phi$ is a well-defined measurable real-valued function and the integral of $\phi$ over any collection of its variables (w.r.t. $P_X$ for the $x$-variables and $P_Y$ for the $y$-variables) is finite.   Thus we have proven the  integrability properties required by Theorem \ref{thm:McDiarmid_general_mean_bound_method}.

 Writing $y^k\coloneqq(y_{k,1},...,y_{k,m})$ we have
\begin{align}
&|\phi(x_1,y^1,...,x_k,y^k,...,x_n,y^n)-\phi(x_1,y^1,...,\tilde{x}_k,y^k,...,x_n,y^n)|\\
\leq&\sup_{g\in\mathcal{G}}\left|\frac{1}{nm}\sum_{j=1}^mg(x_k,y_{k,j})-\frac{1}{nm}\sum_{j=1}^mg(\tilde{x}_k,y_{k,j})\right|
\leq \frac{1}{nm}\sum_{j=1}^m h_{\mathcal{X}}(x_k,\tilde{x}_k,y_{k,j})\notag
\end{align}
and similarly
\begin{align}
&|\phi(x_1,y^1,...,x_k,y^k,...,x_n,y^n)-\phi(x_1,y^1,...,x_k,y_{k,1},...,\tilde{y}_{k,\ell},...,y_{k,m},...,x_n,y^n)|\\
\leq&\frac{1}{nm}h_{\mathcal{Y}}(y_{k,\ell},\tilde{y}_{k,\ell})\,.\notag
\end{align}
Therefore, defining $P^{n,m}_{k,\ell}\coloneqq P_Y^{m-\ell}\times  (P_X\times P_Y^m)^{n-k}$, $\ell=0,...,m-1$, and $P^{n,m}_{k,m}\coloneqq  (P_X\times P_Y^m)^{n-k}$,  we have
\begin{align}
&\left|E_{P^{n,m}_{k,0}}[\phi(x_1,y^1,...,x_k,\cdot)]-E_{P^{n,m}_{k,0}}[\phi(x_1,y^1,...,\tilde{x}_k,\cdot)]\right|\\
\leq&E_{P^{n,m}_{k,0}}\left[\left|\phi(x_1,y^1,...,x_k,\cdot)-\phi(x_1,y^1,...,\tilde{x}_k,\cdot)\right|\right]\leq \frac{1}{n}   E_{P_Y}[  h_{\mathcal{X}}(x_k,\tilde{x}_k,\cdot)]\notag
\end{align}
and
\begin{align}
&\left|E_{P^{n,m}_{k,\ell}} [\phi(x_1,y^1,...,x_k,y_{k,1},...,{y}_{k,\ell},\cdot)]-E_{P^{n,m}_{k,\ell}}[\phi(x_1,y^1,...,x_k,y_{k,1},...,\tilde{y}_{k,\ell},\cdot)]\right|\\
\leq&E_{P^{n,m}_{k,\ell}}\left [\left|\phi(x_1,y^1,...,x_k,y_{k,1},...,{y}_{k,\ell},\cdot)-\phi(x_1,y^1,...,x_k,y_{k,1},...,\tilde{y}_{k,\ell},\cdot)\right|\right]\notag\\
\leq &\frac{1}{nm}h_{\mathcal{Y}}(y_{k,\ell},\tilde{y}_{k,\ell})\,.\notag
\end{align}
Thus we have obtained the necessary one-component mean-difference bounds.

Let $K=nK_X$.  For $\lambda\in(-K,K)$, Assumption \ref{assump:ULLN} implies    $-K_Y<\lambda/(nm)<K_Y$ and 
\begin{align}
&\int  \cosh\left(\lambda  n^{-1} E_{P_Y}[h_{\mathcal{X}}(x,\tilde{x},\cdot)]\right)  (P_X\times P_X)(dxd\tilde{x}) \leq e^{\xi_X(|\lambda|/n)}\,,\\
&\int \cosh\left(\lambda (nm)^{-1} h_{\mathcal{Y}}(y,\tilde{y})\right)  (P_Y\times P_Y)(dyd\tilde{y})\leq e^{\xi_Y(|\lambda |/(nm))}\,.\notag
\end{align}
 Therefore  we have shown that all of the hypotheses of Theorem \ref{thm:McDiarmid_general_mean_bound_method}  hold  and hence we obtain $\phi\in L^1({P^{n,m}})$ and, after changing variables from $\lambda$ to $n\lambda$ in the supremum in \eqref{eq:new_McDiarmid_conc}, we  also obtain   \eqref{eq:new_ULLN_conc}.
\end{proof}

\subsection{Bounding the Mean}
For  appropriate  families of functions, one can use Rademacher complexity and entropy integral techniques to show that $E_{{P^{n,m}}}[\phi]$ in \eqref{eq:new_ULLN_conc} approaches zero as  the number of samples $n\to\infty$. For this purpose, we will use  the following novel distribution-dependent Rademacher complexity bound, which is capable of producing finite bounds for distributions with unbounded support, without needing to employ truncation arguments.  A less general version of this result (i.e., for $\eta=0$)  was previously proven in \cite{birrell2025statisticalerrorboundsgans}; we  provide a detailed proof in Appendix \ref{app:rademacher_bound},  Theorem \ref{thm:Rademacher_bound_unbounded_support}.
\begin{theorem}\label{thm:Rademacher_bound_unbounded_support_main}
Let $(\Theta,d_\Theta)$ be a non-empty pseudometric space and denote its diameter by $D_\Theta\coloneqq\sup_{\theta_1,\theta_2\in\Theta}d_\Theta(\theta_1,\theta_2)$.  Let $(\mathcal{Z},P)$ be a probability space and suppose we have a collection of measurable   functions $g_\theta:\mathcal{Z}\to \mathbb{R}$, $\theta\in\Theta$. Define $\mathcal{G}\coloneqq \{g_\theta:\theta\in\Theta\}$.

If for all $z\in\mathcal{Z}$  there exists $L(z)\in(0,\infty)$ such that $\theta\mapsto g_\theta(z)$ is $L(z)$-Lipschitz then
  for all $n\in\mathbb{Z}^+$ and all $z\in\mathcal{Z}^n$ we have the following bound on the empirical Rademacher complexity, as defined in   \eqref{def:Rademacher_complexity}:
\begin{align}
\widehat{\mathcal{R}}_{\mathcal{G},n}(z)\leq&L_n(z)\inf_{\eta>0}\left\{4\eta+1_{\eta<  D_\Theta}8\sqrt{2}n^{-1/2}\int_{\eta/2}^{ D_\Theta/2}\sqrt{\log  N(\epsilon,\Theta,d_{\Theta})}  d\epsilon\right\}\,, 
\end{align} 
where $L_n(z)\coloneqq \left(\frac{1}{n}\sum_{i=1}^n L(z_i)^2\right)^{1/2}$ and  $N(\epsilon,\Theta,d_\Theta)$ denotes the size of a minimal $\epsilon$-cover of $\Theta$ with respect to the pseudometric $d_\Theta$. Assuming the requisite measurability properties, we also obtain a bound on the $P$-Rademacher complexity  \eqref{eq:Rademacher_def},
\begin{align}
\mathcal{R}_{\mathcal{G},P,n}\leq   E_P[L^2]^{1/2}\inf_{\eta>0}\left\{4\eta+1_{\eta<  D_\Theta}8\sqrt{2}n^{-1/2}\int_{\eta/2}^{ D_\Theta/2}\sqrt{\log  N(\epsilon,\Theta,d_{\Theta})}  d\epsilon\right\}\,.
\end{align}
 
Moreover, if $\Theta$ is the unit ball in $\mathbb{R}^k$ under some norm and $d_\Theta$ is the corresponding metric  then  we have
\begin{align}\label{eq:entropy_int_bound_unit_ball_main}
\inf_{\eta>0}\left\{4\eta+1_{\eta<  D_\Theta}8\sqrt{2}n^{-1/2}\int_{\eta/2}^{ D_\Theta/2}\sqrt{\log  N(\epsilon,\Theta,d_{\Theta})}  d\epsilon\right\}  \leq 32 (k/n)^{1/2}\,.
\end{align}
\end{theorem}
\begin{remark}
{ Pseudometrics allow for the possibility of $\theta_1\neq \theta_2$ but $d_\Theta(\theta_1,\theta_2)=0$; see, e.g.,  Remark 10.5 in \cite{anthony1999neural}.}
\end{remark}
We emphasize that the primary novelty of the above Rademacher complexity bound is the ability to use a sample-dependent Lipschitz constant, $L(z)$, which is only required to be square integrable.  This allows the result to be applied to distributions with unbounded support,  even in certain heavy-tailed cases.  Many commonly used neural network architectures satisfy a local Lipschitz bound of the form assumed above with $\mathcal{Z}=\mathbb{R}^d$, $\Theta$ a bounded subset of $\mathbb{R}^k$, and $L(z)=a+b\|z\|$ for some $a,b$, and so the assumption that $L(z)$ is square integrable  is often satisfied so long as $P$ has finite second moment.

We are now ready to derive a bound on the mean of quantities of the form \eqref{eq:ULLN_phi_def}.  For this, we will work under the following assumption.
\begin{assumption}\label{assump:mean_bound}
Let $(\Theta,d_\Theta)$ be a non-empty pseudometric space, $(\mathcal{X},P_X)$, $(\mathcal{Y},P_Y)$ be   probability spaces, and $g_\theta:\mathcal{X}\times\mathcal{Y}\to\mathbb{R}$, $\theta\in\Theta$,  be measurable.  Suppose that we have the following:
\begin{enumerate}
\item  $c:\Theta\to\mathbb{R}$ and   $h:\mathcal{X}\times\mathcal{Y}\to[0,\infty)$  such that $h\in L^1(P_X\times P_Y)$ and $|g_\theta-c_\theta|\leq h$ for all $\theta\in\Theta$.
\item  A measurable $L:\mathcal{X}\times\mathcal{Y}\to (0,\infty)$ such that  $\theta\mapsto g_\theta(x,y)-c_\theta$ is $L(x,y)$-Lipschitz  with respect to $d_\Theta$ for all $x\in\mathcal{X}$, $y\in\mathcal{Y}$.
\end{enumerate}
\end{assumption}
\begin{remark}
The purpose of allowing for a nonzero $c_\theta$ is to cancel a possible  parameter that shifts the output of $g_\theta$ (e.g., the bias parameter of the output layer of a neural network).
\end{remark}

\begin{lemma}\label{lemma:variable_reuse_expectation_bound}
Under Assumption \ref{assump:mean_bound},  for $n,m\in\mathbb{Z}^+$ and $\Theta_0$ a countable  subset of $\Theta$ we have { 
\begin{align}\label{eq:mean_bound}
&E_{P^{n,m}}\!\left[\sup_{\theta\in\Theta_0}\left\{\pm\left(\frac{1}{nm}\sum_{i=1}^n  \sum_{j=1}^m g_\theta(x_i,y_{i,j})
-E_{P_X\times P_Y}[ g_\theta]\right)\right\}\right]\leq C_{n,m}\,,
\end{align}
where      $P^{n,m}\coloneqq \prod_{i=1}^n \left(P_X\times \prod_{j=1}^m P_Y\right)$ and
\begin{align}\label{C_nm_def}
C_{n,m}\coloneqq& 2\left((1-1/m)\int  E_{P_Y(dy)}[ L(x,y)] ^2 P_X(dx)+\frac{1}{m} E_{P_X\times P_Y}[L(x,y)^2]\right)^{1/2}\\
&\times\inf_{\eta>0}\left\{4\eta+1_{\eta<D_\Theta}8\sqrt{2}n^{-1/2}\int_{\eta/2}^{D_\Theta/2}\sqrt{\log N(\epsilon,\Theta,d_\Theta)}d\epsilon\right\}\notag\,.
\end{align}}
\end{lemma}
\begin{remark}
Jensen's inequality implies 
\begin{align}
\int  E_{P_Y(dy)}[ L(x,y)] ^2 P_X(dx)\leq E_{P_X\times P_Y}[L(x,y)^2]\,.
\end{align}
Therefore  the right-hand side of \eqref{eq:mean_bound} is non-increasing in $m$, thus demonstrating a quantitative benefit of sample reuse.  When $m$ is large, \eqref{eq:mean_bound}-\eqref{C_nm_def} can be interpreted intuitively as  saying that the $y$ dependence of the Lipschitz constant can be integrated out; this is the behavior one would expect if it was justified to replace the empirical $Y$-average with an expectation. 
\end{remark}
\begin{proof}
Define the probability measure  $P^m\coloneqq P_X\times P_Y^m$ on $\mathcal{X}\times\mathcal{Y}^m$ and define
\begin{align}
\mathcal{G}_0=\left\{z\coloneqq(x,y_1,...,y_m)\mapsto\frac{1}{m}\sum_{j=1}^m (g_\theta(x,y_j)-c_\theta):\theta\in\Theta_0\right\}\,,
\end{align}
a countable family of measurable functions on $\mathcal{X}\times\mathcal{Y}^m$. For $\theta\in\Theta_0$ we have
\begin{align}
\left|\frac{1}{m}\sum_{j=1}^m (g_\theta(x,y_j)-c_\theta)\right|\leq \frac{1}{m}\sum_{j=1}^m |g_\theta(x,y_j)-c_\theta|\leq\frac{1}{m}\sum_{j=1}^m h(x,y_j)\in L^1(P^m)\,.
\end{align}
The ULLN mean bound from Theorem \ref{thm:ULLN_unbounded} in Appendix \ref{app:ULLN_rademacher}   therefore implies
\begin{align}\label{eq:mean_bound_1}
E_{P^{n,m}}\!\left[\sup_{\theta\in\Theta_0}\left\{\pm\left(\frac{1}{nm}\sum_{i=1}^n  \sum_{j=1}^m g_\theta(x_i,y_{i,j})
-E_{P_X\times P_Y}[ g_\theta]\right)\right\}\right]\leq 2\mathcal{R}_{\mathcal{G}_0,P^m,n}\,.
\end{align}
{ Note that the use of a countable $\Theta_0$ ensures there are no measurability issues.}

To bound the Rademacher complexity on the right-hand side of \eqref{eq:mean_bound_1}, let 
\begin{align}
\mathcal{G}\coloneqq\left\{z\coloneqq(x,y_1,...,y_m)\mapsto\frac{1}{m}\sum_{j=1}^m (g_\theta(x,y_j)-c_\theta):\theta\in\Theta\right\}\,.
\end{align}
 For $\theta_1,\theta_2\in\Theta$ we have
\begin{align}
&\left|\frac{1}{m}\sum_{j=1}^m (g_{\theta_1}(x,y_j)-c_{\theta_1})-\frac{1}{m}\sum_{j=1}^m (g_{\theta_2}(x,y_j)-c_{\theta_2})\right|\\
\leq&\frac{1}{m}\sum_{j=1}^m| (g_{\theta_1}(x,y_j)-c_{\theta_1})- (g_{\theta_2}(x,y_j)-c_{\theta_2})\notag|\leq\frac{1}{m}\sum_{j=1}^mL(x,y_j)d_\Theta(\theta_1,\theta_2) \,.\notag
\end{align}
Therefore we can use the empirical Rademacher complexity from Theorem \ref{thm:Rademacher_bound_unbounded_support_main}   to obtain
\begin{align}
\widehat{\mathcal{R}}_{\mathcal{G},n}(z^n)\leq L_n(z^n)\inf_{\eta>0}\left\{4\eta+1_{\eta<D_\Theta}8\sqrt{2}n^{-1/2}\int_{\eta/2}^{D_\Theta/2}\sqrt{\log N(\epsilon,\Theta,d_\Theta)}d\epsilon\right\}\,,
\end{align}
where we use the notation $z^n\coloneqq (z_1,...,z_n)$, $z_i\coloneqq (x_i,y_{i,1},...,y_{i,m})$ and 
\begin{align}
L_n(z^n)\coloneqq \left(\frac{1}{n}\sum_{i=1}^n \left(\frac{1}{m}\sum_{j=1}^mL(x_i,y_{i,j})\right)^2\,\right)^{1/2}\,.
\end{align}
Clearly we have $\widehat{\mathcal{R}}_{\mathcal{G}_0,n}(z^n)\leq\widehat{\mathcal{R}}_{\mathcal{G},n}(z^n)$ and therefore

\begin{align}\label{eq:R_G0_bound}
\mathcal{R}_{\mathcal{G}_0,P^m,n}\leq& E_{P^{n,m}}[L_n(z^n)]\inf_{\eta>0}\left\{4\eta+1_{\eta<D_\Theta}8\sqrt{2}n^{-1/2}\int_{\eta/2}^{D_\Theta/2}\sqrt{\log N(\epsilon,\Theta,d_\Theta)}d\epsilon\right\}\\
\leq& E_{P^{n,m}}[L_n(z^n)^2]^{1/2}\inf_{\eta>0}\left\{4\eta+1_{\eta<D_\Theta}8\sqrt{2}n^{-1/2}\int_{\eta/2}^{D_\Theta/2}\sqrt{\log N(\epsilon,\Theta,d_\Theta)}d\epsilon\right\}\,,\notag
\end{align}
where
\begin{align}\label{eq:Ln_simplification}
E_{P^{n,m}}\!\left[L_n(z^n)^2\right]
=&\frac{1}{m^2} \sum_{j=1}^m\sum_{\ell=1}^mE_{P^m}[L(x,y_{j})L(x,y_{\ell})]\\
=&\frac{1}{m^2} \sum_{j=1}^m\sum_{\ell=1,\ell\neq j}^mE_{P^m}[L(x,y_{j})L(x,y_{\ell})]+\frac{1}{m^2} \sum_{j=1}^m E_{P^m}[L(x,y_{j})^2]\notag\\
=&(1-1/m)\int  E_{P_Y(dy)}[ L(x,y)] ^2 P_X(dx)+\frac{1}{m} E_{P_X\times P_Y}[L(x,y)^2]\,.\notag
\end{align}
Taken together, the results \eqref{eq:mean_bound_1}, \eqref{eq:R_G0_bound},  and \eqref{eq:Ln_simplification} imply the claimed bound \eqref{eq:mean_bound}.
\end{proof}

\subsection{Combining the Concentration Inequality and Mean Bound}
Combining Lemmas \ref{lemma:ULLN_var_reuse}  and \ref{lemma:variable_reuse_expectation_bound} we obtain the following ULLN concentration inequality. 
\begin{theorem}\label{thm:ULLN_var_reuse_mean_bound} 
Let $(\Theta,d_\Theta)$ be a non-empty separable pseudometric space,  $(\mathcal{X},P_X)$, $(\mathcal{Y},P_Y)$ be probability spaces,   $g_\theta:\mathcal{X}\times\mathcal{Y}\to\mathbb{R}$ for $\theta\in\Theta$, and $\mathcal{G}\coloneqq\{g_\theta:\theta\in\Theta\}$ such that Assumptions \ref{assump:ULLN} and  \ref{assump:mean_bound} are satisfied. 
For  $n\in\mathbb{Z}^+$ define $P^{n,m}\coloneqq \prod_{i=1}^n \left(P_X\times \prod_{j=1}^m P_Y\right)$ and  $\phi: \prod_{i=1}^n\left(\mathcal{X}\times \mathcal{Y}^m\right)\to  \mathbb{R}$ by
\begin{align}
&\phi_\pm(z_1,...,z_n)\coloneqq\sup_{\theta\in\Theta}\left\{\pm\left(\frac{1}{nm}\sum_{i=1}^n\sum_{j=1}^mg_\theta(x_i,y_{i,j})-E_{P_X\times P_Y}[g_\theta]\right)\right\}\,,
\end{align}
where $z_i\coloneqq (x_i,y_{i,1},...,y_{i,m})$. Then   $\phi_\pm \in L^1({P^{n,m}})$ and  for all $t\geq 0$ we have
\begin{align}
P^{n,m}(\phi_\pm\geq t+C_{n,m})\leq \exp\left(-n\sup_{\lambda\in[0,K_X)}\left\{\lambda t-\left( \xi_X(\lambda)+m \xi_Y(\lambda/m)\right)\right\}\right)\,,
\end{align}
where $C_{n,m}$ was defined in \eqref{C_nm_def}.

\end{theorem}
\begin{remark}\label{remark:C_nm_decay}
By separability of a pseudometric space, we mean there exists a countable subset  $\Theta_0\subset\Theta$ such that for all $\theta\in\Theta$ there exists a sequence $\theta_j\in\Theta_0$ such that $d_\Theta(\theta_j,\theta)\to 0$ (i.e., a countable dense subset).  When $\Theta$ is the unit ball in $\mathbb{R}^k$ under some norm and $d_\Theta$ is the corresponding metric then a simplified upper bound on $C_{n,m}$ can be obtained by using \eqref{eq:entropy_int_bound_unit_ball_main}, { which leads to $C_{n,m}=O(n^{-1/2})$; see also Appendix \ref{app:mean_ULLN_finite_dim}. }
\end{remark}
\begin{proof}
The result follows from noting that the dominated convergence theorem allows the supremum in the definition of $\phi_+$ to  be restricted to a countable dense subset, $\Theta_0$, of $\Theta$ and then using  Lemmas \ref{lemma:ULLN_var_reuse}  and    \ref{lemma:variable_reuse_expectation_bound}.  We note that $-g_\theta$ satisfies the same set of assumptions, hence  the result holds  for $\phi_-$ as well as for $\phi_+$.
\end{proof}

\section{Concentration Inequalities for Stochastic Optimization Problems with Sample Reuse}\label{sec:stoch_opt}
In this section we apply the techniques developed  above to derive statistical error bounds for stochastic optimization problems, again focusing on the case where we have two types of samples, $X$ and $Y$, with a multiple $m$ more $Y$ samples  than $X$ samples.  More specifically, we derive concentration inequalities for
\begin{align}\label{eq:stoch_opt_error_main}
E_{P_X\times P_Y}[g_{\theta^*_{n,m}}]-\inf_{\theta\in\Theta} E_{P_X\times P_Y}\left[g_\theta\right]\,,
\end{align}
where $g_\theta:\mathcal{X}\times \mathcal{Y}\to\mathbb{R}$ is an appropriate  (unbounded) objective function, depending on parameters  $\theta\in\Theta$, that satisfies certain sample-dependent  (locally) Lipschitz conditions and $\theta^*_{n,m}$ is (approximately) a solution to the empirical optimization problem
\begin{align}
\inf_{\theta\in\Theta}\frac{1}{nm}\sum_{i=1}^n \sum_{j=1}^m g_{\theta}(X_i,Y_{i,j})\,,
\end{align}
with $X_i,Y_{i,j}$ being independent samples from $P_X$ and $P_Y$ respectively. { Thus \eqref{eq:stoch_opt_error_main} is the excess error that results from using the solution to the empirical problem.} The additional assumptions made here are summarized below.
\begin{assumption}\label{assump:X_Y}
Let $(\mathcal{X},P_X)$, $(\mathcal{Y},P_Y)$, $(\Omega,\mathbb{P})$ be  probability spaces, $n,m\in\mathbb{Z}^+$,   and suppose we have random variables $X_i:\Omega\to\mathcal{X}$, $i=1,..,n$, $Y_{i,j}:\Omega\to\mathcal{Y}$, $i=1,...,n$, $j=1,...,m$ that are all  independent  and with $X_i\sim P_X$  and $Y_{i,j}\sim P_Y$. 

Suppose also that $\Theta$ is a measurable space, $g:\Theta\times\mathcal{X}\times\mathcal{Y}\to\mathbb{R}$ is measurable, and there exists $\psi\in L^1(P_X\times P_Y)$ such that $g\geq \psi$.  Let $\epsilon_{n,m}^{\text{opt}}\geq 0$ and suppose we have a $\Theta$-valued random variable $\theta^*_{n,m}$ such that
\begin{align}\label{eq:empirical_optim}
\frac{1}{nm}\sum_{i=1}^n \sum_{j=1}^m g_{\theta^*_{n,m}}(X_i,Y_{i,j})\leq \inf_{\theta\in\Theta}\frac{1}{nm}\sum_{i=1}^n \sum_{j=1}^m g_{\theta}(X_i,Y_{i,j})+\epsilon_{n,m}^{\text{opt}} \,\,\,\,\,\mathbb{P}\text{-a.s.}
\end{align}
\end{assumption} 
\begin{remark}
{ Exact empirical risk minimization algorithms correspond to  $\epsilon^{\text{opt}}_{n,m}=0$ in \eqref{eq:empirical_optim}; we have allowed for  a nonzero $\epsilon_{n,m}^{\text{opt}}$ in order to demonstrate that the optimization error  propagates through to the final bounds in a rather benign way.}
\end{remark}
  We start by deriving a $L^1$ error bound. The  proof is based on a standard error decomposition argument, which we will again use in Theorem \ref{thm:conc_ineq_stoch_opt} below, together with the mean bound from Lemma \ref{lemma:variable_reuse_expectation_bound}.
\begin{theorem}\label{thm:opt_L1_bound}
Let $(\Theta,d_\Theta)$ be a non-empty separable pseudometric space equipped with a $\sigma$-algebra, $(\mathcal{X},P_X)$, $(\mathcal{Y},P_Y)$, $(\Omega,\mathbb{P})$ be  probability spaces,  $g_\theta:\mathcal{X}\times\mathcal{Y}\to\mathbb{R}$, $\theta\in\Theta$, $n,m\in\mathbb{Z}^+$,  and $X_i:\Omega\to\mathcal{X}$, $i=1,..,n$, $Y_{i,j}:\Omega\to\mathcal{Y}$, $i=1,...,n$, $j=1,...,m$ be random variables such that Assumptions \ref{assump:mean_bound} and \ref{assump:X_Y} hold.

Then $g_\theta\in L^1(P_X\times P_Y)$ for all $\theta\in\Theta$, $\inf_{\theta\in\Theta} E_{P_X\times P_Y}[g_\theta]$ is finite,  
and, with $C_{n,m}$ as defined in \eqref{C_nm_def}, we have the $L^1$ error bound
\begin{align}\label{eq:Mean_bound_opt_error}
&\mathbb{E}\left[\left|E_{P_X\times P_Y}[g_{\theta^*_{n,m}}]- \inf_{\theta\in\Theta} E_{P_X\times P_Y}[g_\theta]\right|\right]\leq 2C_{n,m}+\epsilon_{n,m}^{\text{opt}}\,.
\end{align}

\end{theorem}
\begin{remark}
{ Note that  $C_{n,m}\to 0$ as $n\to \infty$ for appropriate spaces $(\Theta,d_\Theta)$ and when the sample-dependent Lipschitz constant from Assumption \ref{assump:mean_bound} is square integrable; see \eqref{C_nm_def} and Remark \ref{remark:C_nm_decay}. We emphasize that our novel Rademacher complexity bound from Theorem \ref{thm:Rademacher_bound_unbounded_support_main} is the key to obtaining a finite $L^1$ error bound  under such relatively weak  assumptions; in particular, it applies to heavy-tailed $P_X$ and $P_Y$  so long as   $L(x,y)$ is square integrable. }

\end{remark}
\begin{proof}
The assumptions imply that $|g_\theta|\leq |g_\theta-c_\theta|+|c_\theta|\leq h+|c_\theta|\in L^1(P_X\times P_Y)$. We also have $g_\theta\geq \psi\in L^1(P_X\times P_Y)$ for all $\theta$ and hence $\inf_\theta E_{P_X\times P_Y}[g_\theta]$ is finite.  On the $\mathbb{P}$-probability one set where  \eqref{eq:empirical_optim} holds we have the following error decomposition:
\begin{align}\label{eq:as_opt_bound}
0\leq &E_{P_X\times P_Y}[g_{\theta^*_{n,m}}]- \inf_\theta E_{P_X\times P_Y}[g_\theta]\\
\leq&E_{P_X\times P_Y}[g_{\theta^*_{n,m}}]-\frac{1}{nm}\sum_{i=1}^n \sum_{j=1}^m g_{\theta^*_{n,m}}(X_i,Y_{i,j})\notag\\
&+\inf_{\theta\in\Theta}\frac{1}{nm}\sum_{i=1}^n \sum_{j=1}^m g_{\theta}(X_i,Y_{i,j})- \inf_\theta E_{P_X\times P_Y}[g_\theta]+\epsilon_{n,m}^{\text{opt}}\notag\\
\leq&\sup_{\theta\in\Theta}\!\left\{\!-\!\left(\!\frac{1}{nm}\sum_{i=1}^n \sum_{j=1}^m g_{\theta}(X_i,Y_{i,j})-E_{P_X\times P_Y}[g_{\theta}]\!\right)\!\!\right\}\notag\\
&+\sup_{\theta\in\Theta}\!\left\{\!\frac{1}{nm}\sum_{i=1}^n \sum_{j=1}^m g_{\theta}(X_i,Y_{i,j})-  E_{P_X\times P_Y}[g_\theta]\!\right\}+\epsilon_{n,m}^{\text{opt}}\,\notag.
\end{align}
The objectives are unchanged under the replacement of $g_\theta$ with $g_\theta-c_\theta$, which the assumptions imply is Lipschitz  in $\theta$, pointwise in $(x,y)$. This, together with the dominated convergence theorem, allows the suprema to be restricted to a countable dense subset  $\Theta_0\subset \Theta$.  Therefore taking expectations and using the bound from   Lemma \ref{lemma:variable_reuse_expectation_bound}  completes the proof.
\end{proof}

Next we use the tools developed above to derive a concentration inequality for stochastic optimization of an unbounded objective.
\begin{theorem}\label{thm:conc_ineq_stoch_opt}
Let $(\Theta,d_\Theta)$ be a non-empty separable pseudometric space equipped with a $\sigma$-algebra, $(\mathcal{X},P_X)$, $(\mathcal{Y},P_Y)$, $(\Omega,\mathbb{P})$ be  probability spaces,  $g_\theta:\mathcal{X}\times\mathcal{Y}\to\mathbb{R}$, $\theta\in\Theta$,   $n,m\in\mathbb{Z}^+$,  and $X_i:\Omega\to\mathcal{X}$, $i=1,..,n$, $Y_{i,j}:\Omega\to\mathcal{Y}$, $i=1,...,n$, $j=1,...,m$ be random variables such that Assumptions  \ref{assump:ULLN} (for $\mathcal{G}\coloneqq\{g_\theta:\theta\in\Theta\}$), \ref{assump:mean_bound}, and \ref{assump:X_Y} all hold.

Then $g_\theta\in L^1(P_X\times P_Y)$ for all $\theta\in\Theta$ and for all $t\geq 0$ we have
\begin{align}
&\mathbb{P}\left(E_{P_X\times P_Y}[g_{\theta^*_{n,m}}]-\inf_{\theta\in\Theta} E_{P_X\times P_Y}[g_\theta]\geq t+2C_{n,m}+\epsilon_{n,m}^{\text{opt}}\right)\\
\leq&\exp\left(-n\sup_{\lambda\in[0,K_X)}\left\{\lambda t/2-(\xi_X(\lambda)+m\xi_Y({\lambda}/{m}))\right\}\right)\,,\notag
\end{align}
{ where $C_{n,m}$ was defined in \eqref{C_nm_def}.}

\end{theorem}
\begin{proof}
The assumptions of Lemma \ref{lemma:variable_reuse_expectation_bound} and Theorem \ref{thm:opt_L1_bound} hold, therefore we combine the  $\mathbb{P}$-a.s. bound \eqref{eq:as_opt_bound}  with the mean bound \eqref{eq:mean_bound} to obtain
\begin{align}\label{eq:opt_prob_bound1}
&\mathbb{P}\left(E_{P_X\times P_Y}[g_{\theta^*_{n,m}}]-\inf_{\theta\in\Theta} E_{P_X\times P_Y}[g_\theta]\geq t+2C_{n,m}+\epsilon_{n,m}^{\text{opt}}\right)\\
\leq &{P^{n,m}}\!\left((\phi_- +\phi_+)/2-E_{P^{n,m}}\!\left[(\phi_-+\phi_+)/2\right]\geq  t/2\right)\,,\label{eq:phi_avg_conc_1}
\end{align}
where $P^{n,m}\coloneqq (P_X\times P_Y^m)^n$,  
\begin{align}
&\phi_\pm(z^n)\coloneqq \sup_{\theta\in\Theta_0}\left\{\pm\left(\frac{1}{nm}\sum_{i=1}^n \sum_{j=1}^m g_{\theta}(x_i,y_{i,j})-E_{P_X\times P_Y}[g_{\theta}] \right)\right\}\,, 
\end{align}
 $z^n\coloneqq(x_1,y_{1,1},...,y_{1,m},...,x_n,y_{n,1},...,y_{n,m})$, and $\Theta_0$ is a countable dense subset of $\Theta$.

While Lemma \ref{lemma:ULLN_var_reuse} does not directly apply to Eq.~\eqref{eq:phi_avg_conc_1}, as the latter involves the average  $\overline{\phi}\coloneqq (\phi_++\phi_-)/2$ and not just one or the other of $\phi_\pm$,  we can still follow the  same proof strategy to arrive at the claimed concentration inequality. We start by obtaining the one-component difference bounds
\begin{align}
&\left|\overline{\phi}(x_1,y^1,...,x_k,y^k,...,x_n,y^n)-\overline{\phi}(x_1,y^1,...,\tilde{x}_k,y^k,...,x_n,y^n)\right|
\leq\frac{1}{nm}\sum_{j=1}^m h_{\mathcal{X}}(x_k,\tilde{x}_k,y_{k,j})\label{eq:phi_bar_bound1}
\end{align}
and
\begin{align}
&\left|\overline{\phi}(x_1,y^1,...,x_k,y^k,...,x_n,y^n)-\overline{\phi}(x_1,y^1,...,x_k,y_{k,1},...,\tilde{y}_{k,\ell},...,y_{k,m},...,x_n,y^n)\right|\label{eq:phi_bar_bound2}\\
\leq&\frac{1}{nm} h_{\mathcal{Y}}(y_{k,\ell},\tilde{y}_{k,\ell})\notag\,.
\end{align}
Recalling the discussion that follows \eqref{eq:one_comp_diff_tildeh_i}, these bounds together with the finiteness and integrability of $h_{\mathcal{X}}$ and $h_{\mathcal{Y}}$ imply that  $\overline{\phi}$ satisfies the integrability conditions required by Theorem \ref{thm:McDiarmid_general_mean_bound_method}.   As in the proof of Lemma \ref{lemma:ULLN_var_reuse}, the one-component mean-difference and MGF bounds required  by Theorem \ref{thm:McDiarmid_general_mean_bound_method}  follow from \eqref{eq:phi_bar_bound1}-\eqref{eq:phi_bar_bound2} together with the bounds
\begin{align}
&\int  \cosh\left(\lambda n^{-1}  E_{P_Y}[h_{\mathcal{X}}(x,\tilde{x},\cdot)]\right)  (P_X\times P_X)(dxd\tilde{x}) \leq e^{\xi_X(|\lambda|/n)}\,,\\
&\int \cosh\left(\lambda  (nm)^{-1} h_{\mathcal{Y}}(y,\tilde{y})\right)  (P_Y\times P_Y)(dyd\tilde{y})
\leq e^{\xi_Y(|\lambda |/(nm))}\,,\notag
\end{align}
which hold for all  $\lambda\in(-nK_X,nK_X)$ by Assumptions  \ref{assump:ULLN}.

Therefore \eqref{eq:new_McDiarmid_conc} from Theorem \ref{thm:McDiarmid_general_mean_bound_method} can be applied to $\overline{\phi}$. After changing variables from $\lambda$ to $n\lambda$ in the supremum, this yields
\begin{align}
&P^{n,m}\left(\overline{\phi}-E_{P^{n,m}}[\overline{\phi}]\geq  s\right)\leq\exp\left(-n\sup_{\lambda\in[0,K_X)}\left\{\lambda s-(\xi_X(\lambda)+m\xi_Y(\lambda/m))\right\}\right)
\end{align}
for all   $s\geq 0$. Letting $s=t/2$ and combining this with  \eqref{eq:opt_prob_bound1}-\eqref{eq:phi_avg_conc_1} completes the proof.
\end{proof}

\section{Applications}\label{sec:applications}  
\subsection{Denoising Score Matching}\label{sec:DSM}
{ Score-based diffusion models utilize a pair of noising/denoising SDEs  in order to train a model capable of (approximately) generating new samples from a target distribution  \cite{songscore}.  More specifically, the goal is to use noised samples from the target distribution to learn an approximation to the so-called score vector field,  $s(x,t)$; the score is used in the  denoising SDE to generate samples.  In practice, the most commonly used formulation of the   score matching problem is  the denoising score matching (DSM)  objective  \cite{vincent2011connection}.  As we will see in Theorem \ref{thm:DSM} below,} DSM is a stochastic optimization problem of the form studied in Theorem \ref{thm:conc_ineq_stoch_opt}, having an unbounded objective function and which also naturally incorporates sample reuse, with the $X$'s  being samples from the target distribution $P_{X_0}$ (with the subscript $0$ used here to emphasize its role as the initial condition for the noising dynamics)  and the $Y$'s being auxiliary Gaussian random variables.  In this section we will show how Theorem \ref{thm:conc_ineq_stoch_opt} can be applied to DSM in order to obtain statistical error bounds, providing a connection between the empirical optimization problem that is used in training and the type of bounds that are assumed  in convergence analyses  \cite{chen2023improved,bentonnearly,li2024towards,li2024sharp,potaptchik2024linear,chen2024equivariant,mimikos2024score}. Such results generally frame their bounds in terms of the score matching error,
\begin{align}\label{eq:score_match_err}
\sum_{j=1}^J \gamma_j{ E}_{ Z_{t_j}}\!\!\left[ \|s^{\theta^*}(\cdot,t_j)-s(\cdot,t_j)\|^2\right]\,,
\end{align}
where $s^{\theta^*}(z,t)$ is the learned score model, $s(z,t)$ is the true score,  $t_j$, $\gamma_j$, $j=1,...,J$ are chosen sequences of timesteps and weights, and $Z_{t_j}$ is the  (linear)  noising dynamics at time $t_j$.  In the following result we use our theory to derive statistical  bounds on the score-matching error \eqref{eq:score_match_err} when the score model $s^{\theta^*}(z,t)$ is trained via empirical DSM.  Despite working under the assumption that the data distribution, $P_{X_0}$, has compact support, our new tools are key for treating the unbounded support of the auxiliary Gaussian random variables that are inherent in the method.

\begin{theorem}\label{thm:DSM}
Suppose we have a data distribution $P_{X_0}$ on $\mathbb{R}^d$ ($\|\cdot\|$ will denote the $2$-norm on $\mathbb{R}^d$),  a non-empty separable pseudometric space $(\Theta,d_\Theta)$  equipped with a $\sigma$-algebra, and a measurable  score model $s:\Theta\times\mathbb{R}^d\times [0,T]\to \mathbb{R}^d$  that satisfy the following.
\begin{enumerate}
\item  We have $q\geq 1$ such that for all $(z,t)$, the map $\theta\mapsto s^\theta(z,t)$ is Lipschitz (with respect to $(d_\Theta,\|\cdot\|)$) with Lipschitz constant $\alpha_t+\beta_t\|z\|^q$.
\item For all $\theta$   we have a linear growth bound $\|s^\theta(z,t)\|\leq a_t+b_t\|z\|$.
\item For all $\theta$ and all $t$, the map $z\mapsto s^\theta(z,t)$ is $L_t$-Lipschitz.
\item  We have $r_0<\infty$ such that $\|x_0\|\leq r_0$ for all $x_0$ in the support of $P_{X_0}$.
\end{enumerate}

Given $c\geq 0$, a continuous noise strength $\sigma_t>0$, and timesteps $t_j$, $j=1,...,J$, define the noising dynamics
$Z_{t}\coloneqq e^{-ct}X_0+\widetilde{\sigma}_{t}Y$, where $X_0\sim P_{X_0}$,  $Y\sim N(0,I_d)$ is independent of $X_0$ and
\begin{align}
\widetilde{\sigma}_t^2\coloneqq\int_0^t e^{-2c(t-s)}\sigma_s^2 ds\,,
\end{align}
along with the  DSM objective
\begin{align}\label{eq:DSM_g_theta}
&g_\theta(x_0,y)\coloneqq \sum_{j=1}^J \gamma_j \|s^{\theta}( e^{-ct_j}x_0+\widetilde{\sigma}_{t_j}y,t_j)+\widetilde{\sigma}^{-1}_{t_j}y\|^2\,.
\end{align}
Suppose  we have  $X_i\sim P_{X_0}$,   and $Y_{i,j}\sim P_Y$, $i=1,...,n$, $j=1,...,m$,  that are all  independent 
 along with a $\Theta$-valued random variable $\theta^*_{n,m}$ that is an approximate minimizer of the empirical DSM problem:
\begin{align}\label{eq:dsm_empirical_optim}
&\frac{1}{nm}\sum_{i=1}^n \sum_{j=1}^m g_{\theta^*_{n,m}}(X_i,Y_{i,j})\leq \inf_{\theta\in\Theta}\frac{1}{nm}\sum_{i=1}^n \sum_{j=1}^m g_{\theta}(X_i,Y_{i,j})+\epsilon_{n,m}^{\text{opt}} \,\,\,\,\,\mathbb{P}\text{-a.s.}\,,
\end{align}
where $\epsilon_{n,m}^{\text{opt}}\geq 0$. 

Then, letting $s(z,t)$ be the true score, for all $\epsilon\geq 0$ we have the following statistical error bounds for DSM:
\begin{align}\label{eq:DSM_conc_ineq}
&\mathbb{P}\left(\sum_{j=1}^J \gamma_j{ E}_{ Z_{t_j}}\!\!\left[ \|s^{\theta^*_{n,m}}(\cdot,t_j)-s(\cdot,t_j)\|^2\right]\geq \epsilon+R_*+2{C}^{\text{DSM}}_{n,m}+\epsilon_{n,m}^{\text{opt}}\right)\\
\leq&\exp\left(-n\sup_{\lambda\in[0,mK_Y)}\left\{\lambda \epsilon/2-A_{X,Y,m}\lambda^2\right\}\right)\,,\notag
\end{align}
where 
\begin{align}
R_*\coloneqq& \inf_{\theta\in\Theta} \sum_{j=1}^J \gamma_j{ E}_{ Z_{t_j}}\!\!\left[ \|s^\theta(\cdot,t_j)-s(\cdot,t_j)\|^2\right]\,,\\
{C}^{\text{DSM}}_{n,m}\coloneqq &2\bigg((1-1/m)\left( A_\Theta+B_\Theta E_{y\sim P_Y}[ \|y\|]+C_\Theta E_{y\sim P_Y}[ \|y\|^q]+D_\Theta E_{y\sim P_Y}[ \|y\|^{q+1}]\right) ^2 \\
&\qquad\qquad\qquad+\frac{1}{m} E_{ y\sim P_Y}\left[(A_\Theta+B_\Theta\|y\|+C_\Theta\|y\|^q+D_\Theta\|y\|^{q+1})^2\right]\bigg)^{1/2}\notag\\
&\times\inf_{\eta>0}\left\{4\eta+1_{\eta<D_\Theta}8\sqrt{2}n^{-1/2}\int_{\eta/2}^{D_\Theta/2}\sqrt{\log N(\epsilon,\Theta,d_\Theta)}d\epsilon\right\}\notag\,,
\end{align}
and the definitions of the various constants are provided in the course of the proof below.

In particular, if $\epsilon\leq 4mK_Y  A_{X,Y,m}$ then the term in the exponent can be simplified via
\begin{align}
&\sup_{\lambda\in[0,mK_Y)}\left\{\lambda \epsilon/2-A_{X,Y,m}\lambda^2\right\}=\frac{\epsilon^2}{16A_{X,Y,m}}\,.\notag
\end{align}
\end{theorem}
\begin{remark}
{ The term $R_*$ is the score-model approximation error.    The statistical error term ${C}^{\text{DSM}}_{n,m}$ can be bounded explicitly in a number of cases.  Specifically, when $\Theta$ is the unit ball in $\mathbb{R}^k$ under some norm  and $d_\Theta$ is the corresponding metric then the entropy integral can be bounded via \eqref{eq:entropy_int_bound_unit_ball_main}, which yields the decay rate $C^{\text{DSM}}_{n,m}=O(n^{-1/2})$. } 
\end{remark}
\begin{proof}
Using the standard method  of rewriting the score-matching objective in DSM form, see \cite{vincent2011connection}, we have
\begin{align}
&\sum_{j=1}^J \gamma_j{ E}_{Z_{t_j}}\left[ \|s^{\theta}(\cdot,t_j)-s(\cdot,t_j)\|^2\right]-R_*=E_{P_{X_0}\times P_Y}\left[g_\theta\right]
-\inf_{\theta\in\Theta}{E}_{P_{X_0}\times P_Y}\left[g_\theta\right]\,,
\end{align}
where the DSM objective $g_\theta$ is given by \eqref{eq:DSM_g_theta}.  

We now show that the assumptions of Theorem \ref{thm:conc_ineq_stoch_opt} hold.  Clearly, Assumption \ref{assump:X_Y} holds for $g_\theta$ and $\theta^*_{n,m}$. Next, by rewriting
\begin{align}
g_\theta(x_0,y)
=& \sum_{j=1}^J \gamma_j \left(\|s^{\theta}( e^{-ct_j}x_0+\widetilde{\sigma}_{t_j}y,t_j)\|^2+2\widetilde{\sigma}^{-1}_{t_j} s^{\theta}( e^{-ct_j}x_0+\widetilde{\sigma}_{t_j}y,t_j)\cdot y+\widetilde{\sigma}^{-2}_{t_j}\|y\|^2\right)
\end{align}
we can compute the growth bound
\begin{align}
&|g_\theta(x_0,y)|\\
\leq&\sum_{j=1}^J \gamma_j \!\left((a_{t_j}+b_{t_j}\| e^{-ct_j}x_0+\widetilde{\sigma}_{t_j}y\|)^2+2\widetilde{\sigma}^{-1}_{t_j}(a_{t_j}+b_{t_j}\| e^{-ct_j}x_0+\widetilde{\sigma}_{t_j}y\|)\|y\|+\widetilde{\sigma}^{-2}_{t_j}\|y\|^2\right)\notag\\
\leq& \sum_{j=1}^J \gamma_j \!\left((a_{t_j}+b_{t_j}( r_0e^{-ct_j}+\widetilde{\sigma}_{t_j}\|y\|)^2+2\widetilde{\sigma}^{-1}_{t_j}(a_{t_j}+b_{t_j}( r_0e^{-ct_j}+\widetilde{\sigma}_{t_j}\|y\|)\|y\|+\widetilde{\sigma}^{-2}_{t_j}\|y\|^2\right)\notag\\
&\coloneqq h(x_0,y)\in L^1(P_{X_0}\times P_Y)\,,\notag
\end{align}
 as well as the Lipschitz bound  
\begin{align}
&|g_{\theta_1}(x_0,y)-g_{\theta_2}(x_0,y)|\\
\leq&\sum_{j=1}^J \gamma_j (\|s^{\theta_1}( e^{-ct_j}x_0+\widetilde{\sigma}_{t_j}y,t_j)\|+\|s^{\theta_2}( e^{-ct_j}x_0+\widetilde{\sigma}_{t_j}y,t_j)\|)\notag\\
&\qquad\times\|s^{\theta_1}( e^{-ct_j}x_0+\widetilde{\sigma}_{t_j}y,t_j)-s^{\theta_2}( e^{-ct_j}x_0+\widetilde{\sigma}_{t_j}y,t_j)\|\notag\\
&+2\sum_{j=1}^J \gamma_j \widetilde{\sigma}^{-1}_{t_j} \|y\| \|s^{\theta_1}( e^{-ct_j}x_0+\widetilde{\sigma}_{t_j}y,t_j)- s^{\theta_2}( e^{-ct_j}x_0+\widetilde{\sigma}_{t_j}y,t_j)\|\notag\\
\leq&2\sum_{j=1}^J \gamma_j( a_{t_j}+b_{t_j}\| e^{-ct_j}x_0+\widetilde{\sigma}_{t_j}y\|+ \widetilde{\sigma}^{-1}_{t_j} \|y\|) ( \alpha_{t_j}+\beta_{t_j}\|  e^{-ct_j}x_0+\widetilde{\sigma}_{t_j}y\|^q)d_\Theta(\theta_1,\theta_2)\notag\\
\leq&2\sum_{j=1}^J \gamma_j( a_{t_j}+r_0b_{t_j}  e^{-ct_j}+(b_{t_j} \widetilde{\sigma}_{t_j}+ \widetilde{\sigma}^{-1}_{t_j})\|y\|) ( \alpha_{t_j}+2^{q-1}\beta_{t_j}(  r_0^qe^{-cqt_j}+\widetilde{\sigma}^q_{t_j}\|y\|^q))d_\Theta(\theta_1,\theta_2)\notag\\
\coloneqq &(A_\Theta+B_\Theta\|y\|+C_\Theta\|y\|^q+D_\Theta\|y\|^{q+1}) d_\Theta(\theta_1,\theta_2)\,.\notag
\end{align}
Therefore Assumption \ref{assump:mean_bound} holds for $g_\theta$ by taking $c_\theta\coloneqq 0$ and  $L(x_0,y)\coloneqq A_\Theta+B_\Theta\|y\|+C_\Theta\|y\|^q+D_\Theta\|y\|^{q+1}$.

Finally, we show that Assumption \ref{assump:ULLN} holds. To do this, first compute
\begin{align}
&|g_\theta(x_0,y)-g_\theta({x}_0,\tilde{y})|\\
\leq&\sum_{j=1}^J \gamma_j  ( \|s^{\theta}( e^{-ct_j}x_0+\widetilde{\sigma}_{t_j}y,t_j)+\widetilde{\sigma}^{-1}_{t_j}y\|+ \|s^{\theta}( e^{-ct_j}x_0+\widetilde{\sigma}_{t_j}\tilde y,t_j)+\widetilde{\sigma}^{-1}_{t_j}\tilde y\|)\notag\\
&\times\|s^{\theta}( e^{-ct_j}x_0+\widetilde{\sigma}_{t_j}y,t_j)-s^{\theta}( e^{-ct_j}x_0+\widetilde{\sigma}_{t_j}\tilde y,t_j)+\widetilde{\sigma}^{-1}_{t_j}(y-\tilde y)\|\notag\\
\leq& \sum_{j=1}^J \gamma_j  ( 2(a_{t_j}+r_0b_{t_j} e^{-ct_j})+(b_{t_j}\widetilde{\sigma}_{t_j}+\widetilde{\sigma}^{-1}_{t_j})(\|y\| +\|\tilde y\|))( L_{t_j}\widetilde{\sigma}_{t_j} 
+\widetilde{\sigma}^{-1}_{t_j})\|y-\tilde y\|\notag\\
&\coloneqq (A_\mathcal{Y}+B_{\mathcal{Y}}(\|y\|+\|\tilde{y}\|))\|y-\tilde{y}\|\coloneqq h_{\mathcal{Y}}(y,\tilde{y})\in L^1(P_Y\times P_Y)\label{eq:DSM_hy_def}
\end{align}
and
\begin{align}
&|g_\theta(x_0,y)-g_\theta(\tilde{x}_0,y)|\\
\leq&  \sum_{j=1}^J \gamma_j (\|s^{\theta}( e^{-ct_j}x_0+\widetilde{\sigma}_{t_j}y,t_j)+\widetilde{\sigma}^{-1}_{t_j}y\|+\|s^{\theta}( e^{-ct_j}\tilde{x}_0+\widetilde{\sigma}_{t_j}y,t_j)+\widetilde{\sigma}^{-1}_{t_j}y\|)\notag\\
&\times\|s^{\theta}( e^{-ct_j}x_0+\widetilde{\sigma}_{t_j}y,t_j)-s^{\theta}( e^{-ct_j}\tilde{x}_0+\widetilde{\sigma}_{t_j}y,t_j)\|\notag\\
\leq& 4 r_0\sum_{j=1}^J \gamma_j  L_{t_j}e^{-ct_j} (a_{t_j}+b_{t_j}e^{-ct_j}r_0+(b_{t_j}\widetilde{\sigma}_{t_j}+\widetilde{\sigma}^{-1}_{t_j})\|y\|)\notag\\
&\coloneqq A_{\mathcal{X}}+B_{\mathcal{X}}\|y\|\coloneqq h_{\mathcal{X}}(x,\tilde{x},y)\in L^1(P_{X_0}\times P_{X_0}\times P_Y)\,. \label{eq:DSM_hX_def}
\end{align}

Letting $d\zeta$ be the uniform distribution on $\{-1,1\}$ we can use the triangle inequality and the Cauchy-Schwarz inequality to compute
\begin{align}
&\left(\int |\zeta h_{\mathcal{Y}}(y,\tilde{y})|^pd\zeta dP_YdP_Y\right)^{1/p}\\
\leq&A_\mathcal{Y}\|\|y-\tilde{y}\|\|_{L^p(P_Y\times P_Y)}+B_{\mathcal{Y}}\|(\|y\|+\|\tilde{y}\|)\|y-\tilde{y}\|\|_{L^p(P_Y\times P_Y)}\notag\\
\leq&2A_\mathcal{Y}\|\|y\|\|_{L^p(P_Y)}+2B_{\mathcal{Y}}\|\|y\|\|_{L^{2p}(P_Y)}\|\|y-\tilde{y}\|\|_{L^{2p}(P_Y\times P_Y)}\notag\\
\leq&2A_\mathcal{Y}\|\|y\|\|_{L^p(P_Y)}+4B_{\mathcal{Y}}\|\|y\|\|_{L^{2p}(P_Y)}^2\,.\notag
\end{align}
It is straightforward to verify that $\|y\|$ is sub-Gaussian under $P_Y$.  Therefore    Proposition 2.5.2 in \cite{vershynin2018high} to implies  existence of $M\in[0,\infty)$ such that
\begin{align}
\left(\int |\zeta h_{\mathcal{Y}}(y,\tilde{y})|^pd\zeta dP_YdP_Y\right)^{1/p}
\leq&2A_\mathcal{Y}M\sqrt{p}+4B_{\mathcal{Y}}\left(M\sqrt{2p}\right)^2\\
\leq&2(A_\mathcal{Y}+4B_{\mathcal{Y}})\max\{M,M^2\}p\,.\notag
\end{align}
Therefore Proposition 2.7.1 in \cite{vershynin2018high}  implies $\zeta h_{\mathcal{Y}}(y,\tilde{y})$ is sub-exponential under $d\zeta dP_YdP_Y$ and  there exists a universal constant $C$ such that
\begin{align}\label{eq:DSM_hY_subexp}
\int \cosh(\lambda h_{\mathcal{Y}}(y,\tilde{y}))(P_Y\times dP_Y)(dyd\tilde{y})=&\int e^{\lambda\zeta h_{\mathcal{Y}}(y,\tilde{y})}d\zeta dP_YdP_Y\\
\leq& e^{4C^2(A_\mathcal{Y}+4B_{\mathcal{Y}})^2\max\{M,M^2\}^2\lambda^2}\notag
\end{align}
for all $|\lambda|\leq 1/(2C(A_\mathcal{Y}+4B_{\mathcal{Y}})\max\{M,M^2\})$, i.e., there exists ${\sigma}_Y$ (depending only on  $P_Y$) such that
\begin{align}
\int e^{\lambda\zeta h_{\mathcal{Y}}(y,\tilde{y})}d\zeta dP_YdP_Y\leq e^{4(A_\mathcal{Y}+4B_{\mathcal{Y}})^2 {\sigma}_Y^2\lambda^2}\coloneqq e^{\xi_Y(|\lambda|)}
\end{align}
for all $|\lambda|\leq 1/(2(A_\mathcal{Y}+4B_{\mathcal{Y}}){\sigma}_Y)\coloneqq K_Y$.

Finally compute
\begin{align}
\int  \cosh\left(\lambda  E_{P_Y}[h_{\mathcal{X}}(x,\tilde{x},\cdot)]\right)  (P_X\times P_X)(dxd\tilde{x}) 
\leq&\exp\left(\lambda^2  ( A_{\mathcal{X}}+B_{\mathcal{X}}E_{P_Y}[\|y\|])^2/2\right)\coloneqq e^{\xi_X(|\lambda|)}
\end{align}
for all $\lambda\in\mathbb{R}$; in particular, we can pick $K_X=mK_Y$. This completes the proof that Assumption   \ref{assump:ULLN} holds for $g_\theta$.  Therefore we can apply Theorem \ref{thm:conc_ineq_stoch_opt}; substituting in the definitions of all the relevant terms and defining
\begin{align}
A_{X,Y,m}\coloneqq  ( A_{\mathcal{X}}+B_{\mathcal{X}}E_{P_Y}[\|y\|])^2/2
+4(A_\mathcal{Y}+4B_{\mathcal{Y}})^2 {\sigma}_Y^2/m\,,
\end{align}
 we arrive at the claimed bound \eqref{eq:DSM_conc_ineq}.
 \end{proof}
We note that  the assumptions  on the score model in Theorem \ref{eq:score_match_err}, specifically the Lipschitz and growth bounds, do not apply to all score models used in practice, though they can be enforced by appropriate activation and architectural choices, including spectral normalization \cite{miyato2018spectral}.   We require these assumptions for two primary reasons: 1) To ensure  $h_{\mathcal{Y}}$, defined in \eqref{eq:DSM_hy_def}, satisfies the  sub-exponential MGF bound \eqref{eq:DSM_hY_subexp}.  2) To enable use of the   Rademacher complexity bound from Theorem \ref{thm:Rademacher_bound_unbounded_support_main}. The assumption that the data distribution has compact support is realistic for some applications (e.g., image generation), though it would certainly be preferable if it could be removed.  Whether the  tools developed here can be  adapted so as to weaken any of the aforementioned assumptions is a question for future work. 

 To the best of the author's knowledge, the present work was the first to  obtain statistical guarantees on DSM in the form of concentration inequalities pertaining to a solution $\theta^*_{n,m}$ to the empirical DSM problem as well as to quantify the benefits of training sample reuse, i.e., the use of many Gaussian samples per training sample.    The inherently unbounded support of the distribution $P_Y$ along with an objective which is unbounded and only locally-Lipschitz  makes this a difficult problem for previous methods, even the more recent works   \cite{kontorovich2014concentration,maurer2021concentration} which do cover some unbounded objectives.  Specifically, both Theorem 1 in \cite{kontorovich2014concentration} and  Theorems 9 and 11   in \cite{maurer2021concentration} require bounds on each component that are uniform in the other variables, while the DSM objective \eqref{eq:DSM_g_theta} does not satisfy this due to  $y$  being inherently unbounded in \eqref{eq:DSM_hX_def}. The  general form of one-component difference bounds treated in our Theorem \ref{thm:McDiarmid_general_mean_bound_method} was key to overcoming this obstacle in Theorem \ref{thm:DSM}.  In addition, the novel Rademacher complexity bound from Theorem \ref{thm:Rademacher_bound_unbounded_support_main} is a crucial tool for obtaining fully explicit results when working with unbounded objective functions, an aspect of the problem that was not considered in  \cite{kontorovich2014concentration,maurer2021concentration}.

\subsection{Generative Adversarial Networks}\label{sec:GANs}
{ Lastly, we show how our results can  be applied to GANs, another class of methods for training a model capable of generating samples from a target distribution.  For simplicity, here we will consider GANs defined in terms of an integral probability metric (IPM), which includes the oft-used Wasserstein GAN \cite{WGAN,wgan:gp,miyato2018spectral}, though other variants are also of interest \cite{GAN,f_GAN,pmlr-v70-arora17a,Bridging_fGan_WGan,pantazis2022cumulant,birrell2022f,birrell2023functionspace}. The goal  here is to solve
\begin{align}\label{eq:IPM_GAN}
\inf_{\theta\in \Theta} d_\Gamma(P_X,(G_\theta)_{\#}P_Y)\,,
\end{align}
where $P_X$ is the distribution of the data, $P_Y$ is a chosen (simple) noise source (e.g., a Gaussian), $\Gamma$ is the chosen space of test functions on $\mathcal{X}$ (i.e., the discriminators),  $(G_\theta)_{\#}$ denotes the pushforward by  the generator, $G_{\theta}:\mathcal{Y}\to\mathcal{X}$, parameterized by $\theta\in\Theta$, and $d_\Gamma$ denotes the $\Gamma$-IPM, defined by
\begin{align}\label{eq:IPM_def}
d_\Gamma(P,Q)\coloneqq \sup_{\gamma\in\Gamma}\{E_P[\gamma]-E_Q[\gamma]\}\,.
\end{align}
An optimal $\theta^*$ in \eqref{eq:IPM_GAN} leads to the corresponding generative model $(G_{\theta^*})_{\#}P_Y$, which can be used as an approximation to the target distribution, $P_X$.

Our goal is to quantify the excess statistical error incurred by (approximately) solving  the empirical GAN problem
\begin{align}\label{eq:empirical_IPM_GAN_intro}
 \inf_{\theta\in\Theta}d_\Gamma(P_{X,n},(G_{\theta})_{\#}P_{Y,n,m})\,,
\end{align} 
 where $P_{X,n}$ and $P_{Y,n,m}$ are the empirical distributions of the independent samples $X_i$, $i=1,...,n$ and  $Y_{i,j}$, $i=1,...,n$, $j=1,...,m$ from $P_X$ and $P_Y$ respectively.  As we will see below, this application is simpler than DSM in some ways, primarily due to the lack of interaction between the $X$ and $Y$ samples in the objective (i.e., the objective is a sum of terms that depend only on $X$ and terms that depend only on $Y$), though  the min-max form of the problems \eqref{eq:IPM_GAN} and \eqref{eq:empirical_IPM_GAN_intro}  is an additional complication, as it is distinct from the problem type  considered in Section \ref{sec:stoch_opt}.  However,   error decomposition arguments similar to those used in earlier analyses of GANs such as  \cite{huang2022error} will allow us to convert to a problem to which our general tools do apply.     The key new ideas required for this conversion are already present in  the context of an easier $L^1$ error bound, which we present first. We work under the following assumptions.
}

\begin{assumption}\label{assump:GAN}
Let $(\mathcal{X},P_X)$, $(\mathcal{Y},P_Y)$ be   probability spaces and $\Phi,\Theta$ be measurable spaces, $d_\Phi$ be a separable pseudo-metric on $\Phi$, $d_{\Phi\times\Theta}$ be a separable pseudometric on $\Phi\times\Theta$. Suppose we have the following:
\begin{enumerate}
\item A measurable map  $\gamma:\Phi\times \mathcal{X}\to\mathbb{R}$ (i.e., the discriminators).
\item A measurable map $G:\Theta\times \mathcal{Y}\to \mathcal{X}$ (i.e., the generator).
\item  $c:\Phi\to\mathbb{R}$, $\psi_{\mathcal{X}}:\mathcal{X}\to[0,\infty)$, $\psi_{\mathcal{Y}}:\mathcal{Y}\to[0,\infty)$ such that $\psi_{\mathcal{X}}\in L^1(P_X)$, $\psi_{\mathcal{Y}}\in L^1(P_Y)$, $|\gamma_{\phi}-c_{\phi}|\leq\psi_{\mathcal{X}}$ for all $\phi\in\Phi$ and $|\gamma_{\phi}\circ G_\theta-c_{\phi}|\leq\psi_{\mathcal{Y}}$  for all ${\phi}\in{\Phi}$, $\theta\in{\Theta}$.
\item  Measurable $L_{\mathcal{X}}:\mathcal{X}\to (0,\infty)$ and $L_{\mathcal{Y}}:\mathcal{Y}\to (0,\infty)$  such that  $\phi\mapsto \gamma_\phi(x)-c_\phi$ is $L_{\mathcal{X}}(x)$-Lipschitz  with respect to $d_\Phi$ and $(\phi,\theta)\mapsto \gamma_\phi\circ G_\theta(y)-c_\phi$ is $L_{\mathcal{Y}}(y)$-Lipschitz  with respect to $d_{\Phi\times\Theta}$.
\item  $C\geq 0$ such that $d_{\Phi\times\Theta}((\phi_1,\theta),(\phi_2,\theta))\leq  Cd_\Phi(\phi_1,\phi_2)$ for all $\phi_1,\phi_2\in\Phi$, $\theta\in\Theta$.\vspace{.5mm}
\item  $\widetilde{C}\geq 0$ such that $d_\Phi(\phi_1,\phi_2)\leq \widetilde{C}  d_{\Phi\times\Theta}((\phi_1,\theta_1),(\phi_2,\theta_2))$ for all $\phi_1,\phi_2\in\Phi$, $\theta_1,\theta_2\in\Theta$.
\item A  probability space $(\Omega,\mathbb{P})$, $n,m\in\mathbb{Z}^+$,     and  random variables $X_i:\Omega\to\mathcal{X}$, $i=1,..,n$, $Y_{i,j}:\Omega\to\mathcal{Y}$, $i=1,...,n$, $j=1,...,m$ that are all  independent  and with $X_i\sim P_X$  and $Y_{i,j}\sim P_Y$. Denote the corresponding empirical distributions by  $P_{X,n}\coloneqq \frac{1}{n}\sum_{i=1}^n \delta_{X_i}$ and  $P_{Y,n,m}\coloneqq \frac{1}{nm}\sum_{i=1}^n\sum_{j=1}^m \delta_{Y_{i,j}}$.
\item  $\epsilon_{n,m}^{\text{opt}}\geq 0$ and a  $\Theta$-valued random variable   $\theta^*_{n,m}$    that satisfies
\begin{align}\label{eq:empirical_IPM_GAN}
d_{\Gamma_\Phi}(P_{X,n},(G_{\theta^*_{n,m}})_{\#}P_{Y,n,m})\leq \inf_{\theta\in\Theta}d_{\Gamma_\Phi}(P_{X,n},(G_{\theta})_{\#}P_{Y,n,m}) +\epsilon_{n,m}^{\text{opt}} \,\,\,\mathbb{P}\text{-a.s.}\,,
\end{align}
where  $\Gamma_\Phi\coloneqq\{\gamma_\phi:\phi\in\Phi\}$ and $d_{\Gamma_\Phi}$ is as defined in \eqref{eq:IPM_def}.
\end{enumerate}
\end{assumption}

\begin{theorem}\label{thm:GAN_mean_bound}
 Under Assumption \ref{assump:GAN},  we have the $L^1$  error bound
\begin{align} 
& \mathbb{E}\left[\left|d_{\Gamma_\Phi}(P_X,(G_{\theta^*_{n,m}})_{\#}P_Y)-\inf_{\theta\in\Theta} d_{\Gamma_\Phi}(P_X,(G_\theta)_{\#}P_Y)\right|\right]\leq 2C_{n,m}^{\text{GAN}}+\epsilon_{n,m}^{\text{opt}}\,,
\end{align}
where
\begin{align}\label{eq:Cnm_GAN}
C_{n,m}^{\text{GAN}}\coloneqq &2\left((1-1/m)\int  E_{P_Y(dy)}[ L(x,y)] ^2 P_X(dx)+\frac{1}{m} E_{P_X\times P_Y}[L(x,y)^2]\right)^{1/2}\\
&\times\inf_{\eta>0}\left\{4\eta+1_{\eta<D_{\Phi\times \Theta}}8\sqrt{2}n^{-1/2}\int_{\eta/2}^{D_{\Phi\times \Theta}/2}\sqrt{\log N(\epsilon,\Phi\times \Theta,d_{\Phi\times \Theta})}d\epsilon\right\}\,,\notag
\end{align} 
 $L(x,y)\coloneqq \widetilde{C} L_{\mathcal{X}}(x)  +L_{\mathcal{Y}}(y)$, and  the  diameter of     $\Phi\times\Theta$ is denoted by
\begin{align}
D_{\Phi\times \Theta}\coloneqq\sup_{(\phi_1,\theta_1),(\phi_2,\theta_2)\in\Phi\times\Theta}d_{\Phi\times\Theta}((\phi_1,\theta_1),(\phi_2,\theta_2))\,.
\end{align}

\end{theorem}
\begin{remark}\label{remark:GAN_error}
As with the earlier results, the  entropy integral can be bounded in various cases; e.g., when \eqref{eq:entropy_int_bound_unit_ball_main} is applicable one obtains a $C^{\text{GAN}}_{n,m}=O(n^{-1/2})$ bound, provided both $L_{\mathcal{X}}$ and $L_{\mathcal{Y}}$ are square integrable.
\end{remark}
\begin{proof}
We have  $|\gamma_\phi|\leq \psi_{\mathcal{X}}+c_\phi$ and $|\gamma_\phi\circ G_\theta|\leq \psi_{\mathcal{Y}}+c_\phi$, therefore ${\Gamma_\Phi}\subset L^1(P_X)\cap L^1((G_\theta)_{\#}P_Y)$ for all $\theta\in\Theta$.   Moreover,
\begin{align}
|E_{P_X}[\gamma_\phi]-E_{(G_\theta)_{\#}P_Y}[\gamma_\phi]|
\leq&E_{P_X}[|\gamma_\phi-c_\phi|]+E_{(G_\theta)_{\#}P_Y}[|\gamma_\phi-c_\phi|]\\
\leq&E_{P_X}[\psi_{\mathcal{X}}]+E_{P_Y}[\psi_{\mathcal{Y}}]<\infty\,.\notag
\end{align}
Therefore $|d_{\Gamma_\Phi}(P_X,(G_\theta)_{\#}P_Y)|\leq E_{P_X}[\psi_{\mathcal{X}}]+E_{P_Y}[\psi_{\mathcal{Y}}]<\infty$  for all $\theta$ and $\inf_\theta d_{\Gamma_\Phi}(P_X,(G_\theta)_{\#}P_Y)$ is finite.    Letting $\Phi_0$ be  a countable  dense subset of $\Phi$, for $\phi\in\Phi$ and $\phi_j\in\Phi_0$ with $d_\Phi(\phi_j,\phi)\to 0$ we have 
\begin{align}
&|\gamma_{\phi_j}-c_{\phi_j}-(\gamma_\phi-c_\phi)|\leq L_\mathcal{X}d_\Phi(\phi_j,\phi)\to0 \,\,\,\,\text{ as $j\to\infty$,}\\
&|\gamma_{\phi_j}\circ G_\theta-c_{\phi_j}-(\gamma_\phi\circ G_\theta-c_\phi)|\leq L_\mathcal{Y} d_{\Phi\times \Theta}((\phi_j,\theta),(\phi,\theta))\leq  L_\mathcal{Y} C d_{\Phi}(\phi_j,\phi)\to0 \,\,\,\,\text{ as $j\to\infty$.}\notag
\end{align}
Together with the bounds $|\gamma_{\phi}-c_{\phi}|\leq\psi_{\mathcal{X}}$, $|\gamma_{\phi}\circ G_\theta-c_{\phi}|\leq\psi_{\mathcal{Y}}$, the  dominated convergence theorem implies  
\begin{align}d_{\Gamma_\Phi}(P_X,(G_\theta)_{\#}P_Y)= d_{\Gamma_0}(P_X, (G_\theta)_{\#}P_Y)\,,
\end{align}
where $\Gamma_0\coloneqq\{\gamma_\phi:\phi\in\Phi_0\}$. Therefore  $\theta\to d_{\Gamma_\Phi}(P_X,(G_\theta)_{\#}P_Y)$  is measurable. Similarly, we  see that $d_{\Gamma_\Phi}(P_{X,n},(G_{\theta})_{\#}P_{Y,n,m})$ is finite for all $\theta$ and $\inf_{\theta\in\Theta}d_{\Gamma_\Phi}(P_{X,n},(G_{\theta})_{\#}P_{Y,n,m})$ is finite.

By a computation similar to the error decomposition in \cite{huang2022error}, we have the a.s. bound
\begin{align}\label{eq:GAN_error_decomp}
0\leq& d_{\Gamma_\Phi}(P_X,(G_{\theta^*_{n,m}})_{\#}P_Y)-\inf_{\theta\in\Theta} d_{\Gamma_\Phi}(P_X,(G_\theta)_{\#}P_Y)\\
\leq&d_{\Gamma_\Phi}(P_X,(G_{\theta^*_{n,m}})_{\#}P_Y)-d_{\Gamma_\Phi}(P_{X,n},(G_{\theta^*_{n,m}})_{\#}P_{Y,n,m})\notag\\
&+ \inf_{\theta\in\Theta}d_{\Gamma_\Phi}(P_{X,n},(G_{\theta})_{\#}P_{Y,n,m}) -\inf_{\theta\in\Theta} d_{\Gamma_\Phi}(P_X,(G_\theta)_{\#}P_Y) +\epsilon_{n,m}^{\text{opt}}\notag\\
\leq&\sup_{\theta\in\Theta}\left\{d_{\Gamma_\Phi}(P_X,(G_{\theta})_{\#}P_Y)-d_{\Gamma_\Phi}(P_{X,n},(G_{\theta})_{\#}P_{Y,n,m})\right\}\notag\\
&+ \sup_{\theta\in\Theta}\left\{d_{\Gamma_\Phi}(P_{X,n},(G_{\theta})_{\#}P_{Y,n,m}) - d_{\Gamma_\Phi}(P_X,(G_\theta)_{\#}P_Y)\right\} +\epsilon_{n,m}^{\text{opt}}\notag\\
\leq&\sup_{(\phi,\theta)\in\Phi\times\Theta}\left\{-\left(\frac{1}{nm}\sum_{i=1}^n\sum_{j=1}^m g_{\phi,\theta}(X_i,Y_{i,j})- E_{P_X\times P_Y}[g_{\phi,\theta}]\right)\right\}\notag\\
&+ \sup_{(\phi,\theta)\in\Phi\times\Theta}\left\{\frac{1}{nm}\sum_{i=1}^n\sum_{j=1}^m g_{\phi,\theta}(X_i,Y_{i,j})- E_{P_X\times P_Y}[g_{\phi,\theta}]\right\} +\epsilon_{n,m}^{\text{opt}}\notag\,,
\end{align}
where $g_{\phi,\theta}(x,y)\coloneqq\gamma_\phi(x)- \gamma_\phi\circ G_\theta(y)$ is measurable and satisfies
\begin{align}
|g_{\phi,\theta}(x,y)|\leq |\gamma_{\phi}(x)-c_{\phi}|+|\gamma_{\phi}\circ G_{\theta}(y)-c_{\phi}|\leq \psi_\mathcal{X}(x)+\psi_{\mathcal{Y}}(y)\in L^1(P_X\times P_Y)
\end{align}
and
\begin{align}\label{eq:g_phi_theta_Lipschitz}
&  |g_{\phi_1,\theta_1}(x,y)-g_{\phi_2,\theta_2}(x,y)|\\
\leq&    |\gamma_{\phi_1}(x)-c_{\phi_1}-(\gamma_{\phi_2}(x)-c_{\phi_2})|+| \gamma_{\phi_1}\circ G_{\theta_1}(y)-c_{\phi_1}-( \gamma_{\phi_2}\circ G_{\theta_2}(y)-c_{\phi_2})|\notag\\
\leq&\left(\widetilde{C} L_{\mathcal{X}}(x)  +L_{\mathcal{Y}}(y)\right)d_{\Phi\times\Theta}((\phi_1,\theta_1),(\phi_2,\theta_2)) \notag\,.
\end{align}
Let $(\Phi\times\Theta)_0$ be a countable dense subset of $\Phi\times\Theta$. For $(\phi,\theta)\in\Phi\times\Theta$ let $(\phi_j,\theta_j)$ be a sequence in  $(\Phi\times\Theta)_0$ that satisfies  $d_{\Phi\times\Theta}((\phi_j,\theta_j),(\phi,\theta))\to 0$.  Then \eqref{eq:g_phi_theta_Lipschitz} implies
\begin{align}
&|g_{\phi,\theta}(x,y)-g_{\phi_j,\theta_j}(x,y)|\to 0\,\,\,\text{ as $j\to\infty$} 
\end{align}
and hence the dominated convergence theorem implies the suprema on the last two lines of \eqref{eq:GAN_error_decomp} can be restricted to  $(\Phi\times\Theta)_0$ and, after taking expected values, we have
\begin{align}\label{eq:GAN_error_decomp_expected_val}
& \mathbb{E}\left[\left|d_{\Gamma_\Phi}(P_X,(G_{\theta^*_{n,m}})_{\#}P_Y)-\inf_{\theta\in\Theta} d_{\Gamma_\Phi}(P_X,(G_\theta)_{\#}P_Y)\right|\right]\\
\leq&\mathbb{E}\left[\sup_{(\phi,\theta)\in(\Phi\times\Theta)_0}\left\{-\left(\frac{1}{nm}\sum_{i=1}^n\sum_{j=1}^m g_{\phi,\theta}(X_i,Y_{i,j})- E_{P_X\times P_Y}[g_{\phi,\theta}]\right)\right\}\right]\notag\\
&+ \mathbb{E}\left[\sup_{(\phi,\theta)\in(\Phi\times\Theta)_0}\left\{\frac{1}{nm}\sum_{i=1}^n\sum_{j=1}^m g_{\phi,\theta}(X_i,Y_{i,j})- E_{P_X\times P_Y}[g_{\phi,\theta}]\right\}\right] +\epsilon_{n,m}^{\text{opt}}\notag\,.
\end{align}

 Thus we have shown all conditions in Assumption \ref{assump:mean_bound} hold, with $L(x,y)\coloneqq \widetilde{C} L_{\mathcal{X}}(x)  +L_{\mathcal{Y}}(y)$ and $\Theta$ replaced by $\Phi\times\Theta$.  Therefore  the claimed result follows from   Lemma \ref{lemma:variable_reuse_expectation_bound}.
\end{proof}
{ A key feature of Theorem \ref{thm:GAN_mean_bound} is the relatively weak assumptions made on the the distributions $P_X$ and $P_Y$, despite allowing for unbounded discriminators.  In particular, it applies to  heavy-tailed data so long as the discriminators, and their composition with the generator,  are bounded by appropriate integrable envelope functions and satisfy appropriate sample-dependent Lipschitz conditions with square-integrable Lipschitz coefficients. This is in contrast to   the $L^1$ error bound for GANs from Theorem 22 of \cite{huang2022error}  which requires the  data distribution to be sub-exponential.  }
 
Finally, we derive a   concentration inequality for the GAN statistical error.  Here we also illustrate the use of Orlicz-norm bounds, as discussed in Section \ref{sec:Orlicz}.  
\begin{theorem}\label{thm:GAN_conc_ineq}
Under Assumption \ref{assump:GAN}, suppose also that:
\begin{enumerate}
\item We have $h_{\mathcal{X}}\in L^1(P_X\times P_X)$ such that $|\gamma_\phi(x)-\gamma_\phi(\tilde{x})|\leq h_{\mathcal{X}}(x,\tilde{x})<\infty$ for all $x,\tilde{x}\in\mathcal{X}$, $\phi\in\Phi$.
\item We have $h_{\mathcal{Y}}\in L^1(P_Y\times P_Y)$  such that $|\gamma_\phi\circ G_\theta(y)-\gamma_\phi\circ G_\theta(\tilde{y})|\leq h_{\mathcal{Y}}(y,\tilde{y})<\infty$ for all $y,\tilde{y}\in\mathcal{Y}$, $\phi\in\Phi$, $\theta\in\Theta$.
\item  $q_{X}\in[1,\infty)$ and we have the Orlicz norm bound $\|h_{\mathcal{X}}\|_{\Psi_{q_X},P_X^2}\leq C_X<\infty$.
\item $q_{Y}\in[1,\infty)$ and we have the Orlicz norm bound $\|h_{\mathcal{Y}}\|_{\Psi_{q_Y},P_Y^2}\leq C_Y<\infty$.
\end{enumerate}

Then for all $t\geq 0$ we have
\begin{align}\label{eq:GAN_conc_ineq_main}
&\mathbb{P}\left(d_{\Gamma_\Phi}(P_X,(G_{\theta^*_{n,m}})_{\#}P_Y)-\inf_{\theta\in\Theta} d_{\Gamma_\Phi}(P_X,(G_\theta)_{\#}P_Y)\geq t+2C_{n,m}^{\text{GAN}}+\epsilon_{n,m}^{\text{opt}}\right)\\
\leq & \exp\left(-n\sup_{\lambda\in[0,\infty)}\left\{\lambda t/2-(\xi_X(\lambda)+m\xi_Y(\lambda/m)) \right\}\right)\,,\notag
\end{align}
where $C_{n,m}^{\text{GAN}}$ was defined in Eq.~\eqref{eq:Cnm_GAN} and
\begin{align}
\xi_X(\lambda)\coloneqq& \log\left(1+2\sum_{k=1}^\infty (C_X\lambda)^{2k}  \frac{\Gamma(2k/q_X+1) }{(2k)!}\right)\,, \label{eq:GAN_xi_X_def_main}\\\xi_Y(\lambda)\coloneqq & \log\left(1+2\sum_{k=1}^\infty (C_Y\lambda)^{2k}  \frac{\Gamma(2k/q_Y+1) }{(2k)!}\right)\,.\label{eq:GAN_xi_Y_def_main}
\end{align}
\end{theorem}
\begin{remark}
The term $\inf_{\theta\in\Theta} d_{\Gamma_\Phi}(P_X,(G_\theta)_{\#}P_Y)$ is the approximation error inherent in the chosen generator family and noise source.  Thus the above concentration inequality provides a high-probability  bound on the excess statistical error; see also Remark \ref{remark:GAN_error}. Note that $m\xi_Y(\lambda/m)=O(1/m)$ and so the term on the right-hand side of \eqref{eq:GAN_conc_ineq_main} coming from the $Y$ samples is negligible when $m$ is large.   The right-hand side of \eqref{eq:GAN_conc_ineq_main} can be simplified by mirroring the steps used to obtain Corollary \ref{cor:subexp_subGauss_result}. Note that $\Gamma$ in \eqref{eq:GAN_xi_X_def_main}-\eqref{eq:GAN_xi_Y_def_main} denotes the gamma function, not to be confused with the space of discriminators, $\Gamma_\Phi$.
\end{remark}
\begin{proof}
Define  $g_{\phi,\theta}(x,y)\coloneqq\gamma_\phi(x)- \gamma_\phi\circ G_\theta(y)$,  $z^n\coloneqq(x_1,y_{1,1},...,y_{1,m},...,x_n,y_{n,1},...,y_{n,m})$,  and
\begin{align}
&\phi_\pm(z^n) \coloneqq  \sup_{(\phi,\theta)\in(\Phi\times\Theta)_0}\left\{\pm\left(\frac{1}{nm}\sum_{i=1}^n\sum_{j=1}^m g_{\phi,\theta}(x_i,y_{i,j})- E_{P_X\times P_Y}[g_{\phi,\theta}]\right)\right\}
\end{align}
From the proof of Theorem \ref{thm:GAN_mean_bound}, we have $E_{P^{n,m}}[\phi_\pm]\leq C^{\text{GAN}}_{n,m}$.  Combining this with the error decomposition \eqref{eq:GAN_error_decomp} and defining $P^{n,m}\coloneqq (P_X\times P_Y^m)^n$, $\overline{\phi}\coloneqq (\phi_++\phi_-)/2$, we can compute
\begin{align}
&\mathbb{P}\left(d_{\Gamma_\Phi}(P_X,(G_{\theta^*_{n,m}})_{\#}P_Y)-\inf_{\theta\in\Theta} d_{\Gamma_\Phi}(P_X,(G_\theta)_{\#}P_Y)\geq t+2C_{n,m}^{\text{GAN}}+\epsilon_{n,m}^{\text{opt}}\right)\\
\leq &{P}^{n,m}\left(\phi_++\phi_- \geq t+2C_{n,m}^{\text{GAN}}\right)\notag\\
\leq &{P}^{n,m}\left(\overline{\phi} \geq t/2+E_{P^{n,m}}[\overline{\phi}]\right)\,.\notag
\end{align}
Next compute the one-component difference bounds
\begin{align}\label{eq:phi_bar_GAN_bound1}
&\left|\overline{\phi}(x_1,y^1,...,x_k,y^k,...,x_n,y^n)-\overline{\phi}(x_1,y^1,...,\tilde{x}_k,y^k,...,x_n,y^n)\right|\leq\frac{1}{n}h_{\mathcal{X}}(x_k,\tilde{x}_k)
\end{align}
and
\begin{align}\label{eq:phi_bar_GAN_bound2}
&\left|\overline{\phi}(x_1,y^1,...,x_k,y^k,...,x_n,y^n)-\overline{\phi}(x_1,y^1,...,x_k,y_{k,1},...,\tilde{y}_{k,\ell},...,y_{k,m},...,x_n,y^n)\right|\\
\leq&  \frac{1}{nm}h_{\mathcal{Y}}(y_{k,\ell},\tilde{y}_{k,\ell})\,.\notag
\end{align}
Finiteness and integrability of $h_{\mathcal{X}}$ and $h_{\mathcal{Y}}$ together with the discussion following  \eqref{eq:one_comp_diff_tildeh_i} implies that $\overline{\phi}$ satisfies the integrability conditions required by Theorem \ref{thm:McDiarmid_general_mean_bound_method}.  Finally, using Corollary \ref{cor:cosh_bound_Orlicz_Psi_q}, for $\lambda\in\mathbb{R}$ we can bound
\begin{align}
E_{P_X\times P_X}[\cosh(\lambda  h_{\mathcal{X}}/n)]\leq &1+2\sum_{k=1}^\infty (C_X\lambda/n)^{2k}  \frac{\Gamma(2k/q_X+1) }{(2k)!}= e^{\xi_X(|\lambda|/n)} \,,\\
E_{P_Y\times P_Y}[\cosh(\lambda  h_{\mathcal{Y}}/(nm))]\leq &  1+2\sum_{k=1}^\infty (C_Y\lambda/(nm))^{2k}  \frac{\Gamma(2k/q_Y+1) }{(2k)!} = e^{\xi_Y(|\lambda|/(nm))} \,,\notag
\end{align}
where $\xi_X$ and $\xi_Y$ were defined in \eqref{eq:GAN_xi_X_def_main}-\eqref{eq:GAN_xi_Y_def_main}. Therefore the claim follows from applying Theorem \ref{thm:McDiarmid_general_mean_bound_method} with $t$ replaced by $t/2$ and changing variables  from $\lambda$ to $n\lambda$ in the supremum.
\end{proof}

\section{Conclusion}
In  Theorem \ref{thm:McDiarmid_general_mean_bound_method}   we  derived  a   generalization of McDiarmid's inequality that applies to appropriate families of unbounded functions.  Building on the earlier work of   \cite{kontorovich2014concentration} we  obtained  results that cover a significantly  wider class of locally-Lipschitz objective functions, including those that occur in applications such as denoising score matching (DSM).    Using this result, along with the distribution-dependent Rademacher complexity bound in Theorem \ref{thm:Rademacher_bound_unbounded_support_main}, we derived novel uniform law of large numbers results in Lemma \ref{lemma:variable_reuse_expectation_bound} and Theorem \ref{thm:ULLN_var_reuse_mean_bound}  as well as new concentration inequalities for stochastic optimization problems in Theorem \ref{thm:conc_ineq_stoch_opt}, all of which apply to appropriate unbounded (objective) functions and distributions with unbounded support.  We also showed that  our method is able to capture the benefits of sample-reuse in algorithms that  pair training data with  easily-sampled (e.g., Gaussian) auxiliary random variables.  In particular,   in Section \ref{sec:applications} we showed how these results can be used to derive   novel statistical guarantees for both DSM and GANs.

\appendix

\section{Generalization of Theorem \ref{thm:McDiarmid_general_mean_bound_method}}\label{app:McDiarmid_general_mean_bound_method2}
In this appendix we present a slight generalization of the concentration inequality from Theorem \ref{thm:McDiarmid_general_mean_bound_method}.  We do not use this version in any of the results in this paper, but we include it here as it may be of independent interest.  The proof is nearly identical to that of  Theorem \ref{thm:McDiarmid_general_mean_bound_method} and so we omit the details.
\begin{theorem}\label{thm:McDiarmid_general_mean_bound_method2}
Let $(\mathcal{X}_i,P_i)$, $i=1,...,n$ be probability spaces, define $\mathcal{X}^n\coloneqq\prod_{i=1}^n \mathcal{X}_i$ with the product sigma algebra, and let $P^n\coloneqq\prod_{i=1}^n P_i$.  Suppose   $\phi:\mathcal{X}^n\to\mathbb{R}$ satisfies the following:
\begin{enumerate}
\item Integrability:  $\phi\in L^1(P^n)$ and  $\phi(x_1,...,x_i,\cdot)\in L^1(\prod_{j>i} P_j)$ for all $i\in\{1,...,n-1\}$ and all $(x_1,...,x_i)\in\prod_{j=1}^i \mathcal{X}_j$.
\item One-component mean-difference bounds: For all  $i\in\{1,...,n\}$ we have measurable $h_i:\left(\prod_{j<i}\mathcal{X}_j\right)\times\mathcal{X}_i\times\mathcal{X}_i\to[0,\infty]$ such that  
\begin{align}\label{eq:mean_diff_bound_assump2}
\left|E_{\prod_{j>i} P_j}[\phi(x_1,...,x_{i-1},x_i,\cdot)]-E_{\prod_{j>i} P_j}[\phi(x_1,...,x_{i-1},\tilde{x}_i,\cdot)]\right|\leq h_i(x_1,...,x_i,\tilde x_i)
\end{align}
for $(\prod_{j\leq i} P_j)\times P_i$-a.e.   $(x_1,...,x_i,\tilde{x}_i)$. Note that for $i=n$ this assumption should be interpreted as requiring the almost-everywhere bound
\begin{align}\label{eq:mean_diff_bound_assump_n_2}
|\phi(x_1,...,x_{n-1},x_n)-\phi(x_1,...,x_{n-1},\tilde{x}_n)|\leq h_n(x_1,...,x_n,\tilde x_n)\,.
\end{align}
\item MGF bounds: We have    $\xi_i:[0,\infty)\to [0,\infty]$, $i=1,...,n$, such that for all $\lambda\in\mathbb{R}$ the following MGF bounds hold:
\begin{align}\label{eq:cosh_MGF_bound_2}
E_{P_i\times P_i}[\cosh(\lambda h_i(x_1,...,x_{i-1},\cdot))]\leq e^{\xi_i(|\lambda|)} \,\,\,\text{ for  $\prod_{j<i}P_i$-a.e. $(x_1,...,x_{i-1})$\,.}
\end{align}
Note that for $i=1$ this should be interpreted as
\begin{align}
E_{P_1\times P_1}[\cosh(\lambda h_1)]\leq e^{\xi_1(|\lambda|)} \,\,\,\,\text{ for all $\lambda\in\mathbb{R}$.}
\end{align}
\end{enumerate}
  Then  we have the MGF bound
\begin{align}\label{eq:new_McDiarmid_MGF_2}
E_{P^n}\!\left[e^{\lambda (\phi-E_{P^n}[\phi])}\right]\leq \exp\left(\sum_{i=1}^n \xi_i(|\lambda|)\right) \,\,\, \text{  for all $\lambda\in \mathbb{R}$}
\end{align}
and for all $t\geq 0$ we have the concentration inequality
\begin{align}\label{eq:new_McDiarmid_conc_2}
P^n(\pm(\phi-E_{P^n}[\phi])\geq t)\leq \exp\left(-\sup_{\lambda\in[0,\infty)}\left\{\lambda t-\sum_{i=1}^n \xi_i(\lambda)\right\}\right)\,.
\end{align}

\end{theorem}

\section{Proofs of  Results from Section \ref{sec:Orlicz}}\label{app:additional_proofs}
In this appendix we prove the results from Section \ref{sec:Orlicz},  showing how Orlicz norms can be used to produce bounds of the form \eqref{eq:cosh_MGF_bound_2}.
\begin{lemma} 
Let $\Psi$ be a non-decreasing Young function  and   $X$ be a real-valued random variable with $\|X\|_\Psi<\infty$. Then for all $\lambda\in\mathbb{R}$ we have
\begin{align}
\mathbb{E}[\cosh(\lambda X)]\leq&1+\sum_{k=1}^\infty  \frac{\lambda^{2k} }{(2k)!}  \int_0^\infty  \min\left\{\frac{1}{\Psi(t^{1/(2k)}/\|X\|_\Psi)},1\right\}dt\,.
\end{align}
\end{lemma}
\begin{proof}
Note that the claim is trivial if $\|X\|_{\Psi_q}=0$, as that implies $X=0$ a.s.  Assuming  $\|X\|_{\Psi_q}>0$, first recall that  the  concentration inequality 
\begin{align}
\mathbb{P}(|X|> s)\leq \min\left\{\frac{1}{\Psi(s/\|X\|_\Psi)},1\right\}
\end{align}
for all $s>0$; see, e.g., page 96 in  \cite{van2013weak}.  From this, for  $p>0$ we obtain the moment bound
\begin{align}\label{eq:moment_bound_Orlicz}
\mathbb{E}[|X|^p]=\int_0^\infty \mathbb{P}\left(|X|> t^{1/p}\right)dt  \leq \int_0^\infty  \min\left\{\frac{1}{\Psi(t^{1/p}/\|X\|_\Psi)},1\right\}dt\,.
\end{align}
Using \eqref{eq:moment_bound_Orlicz}  to bound the expectation of each term in the Taylor series for $\cosh(\lambda X)$  we obtain the claimed result.
\end{proof}
\begin{corollary} 
Let $(\mathcal{X}_i,P_i)$, $i=1,...,n$ be probability spaces and $\phi:\mathcal{X}^n\to\mathbb{R}$ be measurable. Suppose $\phi$ satisfies $\eqref{eq:one_comp_diff_tildeh_i}$ for real-valued $\tilde{h}_i\in L^1(P_i\times P_i\times\prod_{j>i}P_j)$  and  define $h_i(x_i,\tilde{x}_i)\coloneqq E_{\prod_{j>i} P_j}[\tilde{h}_i(x_i,\tilde{x}_i,\cdot)]$.
\begin{enumerate}
\item If the $h_i$ are sub-Gaussian under $P_i\times P_i$ for all $i$ then for all $t> 0$ we have
\begin{align}
P^n(\pm(\phi-E_{P^n}[\phi])\geq t)\leq& \exp\left(-\frac{t^2}{4\sum_{i=1}^n\|h_i\|_{\Psi_2,P_i^2}^{2}}\right)\,.
\end{align}
\item If the $h_i$ are   sub-exponential under $P_i\times P_i$ for all $i$ then for all $t> 0$ we have
\begin{align}
P^n(\pm(\phi-E_{P^n}[\phi])\geq t)\leq&  \exp\left(-\frac{t^2}{8 \sum_{i=1}^n\|h_i\|_{\Psi_1,P_i^2}^2 +2t \max_i\{\|h_i\|_{\Psi_1,P_i^2}\}}\right) \,.
\end{align}
\end{enumerate}
\end{corollary}
\begin{proof}

First suppose the $h_i$ are    sub-Gaussian. From Corollary \ref{cor:cosh_bound_Orlicz_Psi_q} with $q=2$ we can compute 
\begin{align}
E_{P_i\times P_i}[\cosh(\lambda h_i)]\leq&  1+2\sum_{k=1}^\infty (\|h_i\|_{\Psi_2,P_i^2}\lambda)^{2k}  \frac{k!}{(2k)!}    \\
\leq&1+\sum_{k=1}^\infty (\|h_i\|_{\Psi_2,P_i^2}\lambda)^{2k}  \frac{1}{k!}  \notag\\
=&  \exp((\|h_i\|_{\Psi_2,P_i^2}\lambda)^{2})\,,\notag
\end{align}
where we used the bound $2k!/(2k)!\leq 1/k!$ for all $k\geq 1$.  Therefore the assumptions of Theorem \ref{thm:McDiarmid_general_mean_bound_method} hold with  $\xi_i(\lambda)=(\|h_i\|_{\Psi_2,P_i^2}\lambda)^{2}$ and hence we obtain
\begin{align}
P^n(\pm(\phi-E_{P^n}[\phi])\geq t)\leq& \exp\left(-\sup_{\lambda\in[0,\infty)}\left\{\lambda t-\sum_{i=1}^n(\|h_i\|_{\Psi_2,P_i^2}\lambda)^{2}\right\}\right)\\
=&\exp\left(-\frac{t^2}{4\sum_{i=1}^n\|h_i\|_{\Psi_2,P_i^2}^{2}}\right)\,.\notag
\end{align}

If instead  the $h_i$ are sub-exponential then, noting \eqref{eq:mgf_bound_subexp}, we see that the assumptions of Theorem \ref{thm:McDiarmid_general_mean_bound_method} hold with 
\begin{align}
\xi_i(\lambda)\coloneqq  \frac{2 (\|h_i\|_{\Psi_1,P_i^2}\lambda)^{2}}{1- (\|h_i\|_{\Psi_1,P_i^2}\lambda)^{2}}\,,\,\,\,|\lambda|<1/\|h_i\|_{\Psi_1,P_i^2}\,.
\end{align}
Therefore we obtain
\begin{align} 
P^n(\pm(\phi-E_{P^n}[\phi])\geq t)\leq& \exp\left(-\sup_{0\leq \lambda <1/\max_i\{\|h_i\|_{\Psi_1,P_i^2}\}}\left\{\lambda t-\sum_{i=1}^n \frac{2 (\|h_i\|_{\Psi_1,P_i^2}\lambda)^{2}}{1- (\|h_i\|_{\Psi_1,P_i^2}\lambda)^{2}}\right\}\right)\\
\leq& \exp\left(-\sup_{0\leq \lambda <1/\max_i\{\|h_i\|_{\Psi_1,P_i^2}\}}\left\{\lambda t- \frac{2 \sum_{i=1}^n\|h_i\|_{\Psi_1,P_i^2}^2\lambda^{2}}{1- \max_i\|h_i\|_{\Psi_1,P_i^2}\lambda}\right\}\right) \notag\\
\leq& \exp\left(-\frac{t^2}{8 \sum_{i=1}^n\|h_i\|_{\Psi_1,P_i^2}^2 +2t \max_i\{\|h_i\|_{\Psi_1,P_i^2}\}}\right) \notag\,.
\end{align}
\end{proof}

\section{Distribution-Dependent Rademacher Complexity Bound}\label{app:rademacher_bound} 

 In this appendix we derive a distribution-dependent Rademacher complexity bound that applies to collections of functions that satisfy a  local Lipschitz property with respect to a parameterizing metric space. Finiteness of our bound only requires that the local Lipchitz constant have finite second moment, as opposed to the much stronger assumptions made in other approaches to distribution-dependent bounds, e.g., the sub-Gaussian and sub-exponential  assumptions required in  Proposition 20 of \cite{biau2021some} and Theorem 22 of \cite{huang2022error} respectively.  Distribution-dependent bounds are especially useful in applications such as denoising score matching (DSM), which use auxiliary random variables with known distribution,  despite the distribution of the data being unknown.    A less general version of the following result (i.e., for $\eta=0$)  was previously proven in \cite{birrell2025statisticalerrorboundsgans}.  

First recall the definition of Rademacher complexity.
\begin{definition}\label{def:Rademacher_complexity}
Let $\mathcal{G}$ be a collection of functions on $\mathcal{Z}$ and $n\in\mathbb{Z}^+$.  We define the  empirical Rademacher complexity of $\mathcal{G}$ at  $z\in\mathcal{Z}^n$  by
\begin{align}\label{eq:emp_rad_def}
\widehat{\mathcal{R}}_{\mathcal{G},n}(z)\coloneqq E_\sigma\left[\sup_{g\in\mathcal{G}}\left\{\frac{1}{n}\sigma\cdot g_n(z)\right\}\right]\,,
\end{align}
where $\sigma_i$, $i=1,...,n$  are independent uniform random variables taking values in $\{-1,1\}$, i.e., Rademacher random variables, and $g_n(z)\coloneqq(g(z_1),...,g(z_n))$.  Given a probability distribution $P$ on $\mathcal{Z}$ and assuming that $\widehat{\mathcal{R}}_{\mathcal{G},n}$ is measurable, we define the $P$-Rademacher complexity of $\mathcal{G}$  by
\begin{align}\label{eq:Rademacher_def}
\mathcal{R}_{\mathcal{G},P,n}\coloneqq E_{P^n}\left[\widehat{\mathcal{R}}_{\mathcal{G},n}\right]\,,
\end{align}
where $P^n$ is the $n$-fold product of $P$, i.e., the $z_i$'s become i.i.d. samples from $P$.  
\end{definition}

The Rademacher complexity bound we prove in Theorem \ref{thm:Rademacher_bound_unbounded_support} below will rely on  Dudley's entropy integral; however, we could not locate a standard version   in the literature that suffices for our purposes.  Thus, we will first present a proof of the variant we require, building on Lemma 13.1 in \cite{boucheron2013concentration} as well as Theorem 5.22 in \cite{wainwright2019high}.  First recall the definition of a sub-Gaussian process.
\begin{definition}
Let $(T,d)$ be a pseudometric space and $X_t$, $t\in T$ be mean-zero real-valued random variables on the same probability space.  We call  $X_t$ a sub-Gaussian process with respect to  $d$ if
\begin{align}
E\left[e^{\lambda(X_t-X_{t^\prime})}\right]\leq \exp\left(\frac{\lambda^2d(t,t^\prime)^2}{2}\right)
\end{align}
for all $t,t^\prime\in T$, $\lambda\in\mathbb{R}$. 
\end{definition}
\begin{theorem}\label{thm:Dudley}
Let $X_t$, $t\in T$ be a mean-zero sub-Gaussian process with respect to  a pseudometric $d$ on $T$. Then, assuming the suprema involved are measurable (e.g., if they can be restricted to countable subsets),  we have
\begin{align}\label{eq:Dudley_bound}
E\left[\sup_{t\in T} X_t\right]\leq&E\left[\sup_{t,t^\prime\in T}(X_t-X_{t^\prime})\right]\\
\leq&\inf_{\eta>0}\left\{2E\left[\sup_{d(t,\hat{t})\leq2\eta}(X_t-X_{\hat{t}})\right]+1_{\eta<D_T}8\sqrt{2}\int_{\eta/2}^{D_T/2}\sqrt{\log N(\epsilon,T,d)} d\epsilon\right\}\,,\notag
\end{align}
where $D_T\coloneqq\sup_{t,t^\prime\in T}d(t,t^\prime)$ is the diameter of $(T,d)$.
\end{theorem}
\begin{remark}
The technical features  we  require for our proof of Theorem \ref{thm:Rademacher_bound_unbounded_support} below is that the bound   \eqref{eq:Dudley_bound} is written as the infimum over all $\eta\in(0,\infty)$, where $\eta$ controls the lower bound on the domain of integration (so as to handle  cases where the covering number is not integrable at $0$), of an objective that is non-decreasing in the diameter, $D_T$.  The latter property is necessary for us to make use of the local Lipschitz bound that is the key to our argument in Theorem \ref{thm:Rademacher_bound_unbounded_support}.
\end{remark}
\begin{proof}
The $X_t$'s have mean zero, hence we  have
\begin{align}
0\leq E\left[\sup_t X_t\right]=E\left[\sup_t (X_t-X_{t^\prime})\right]\leq E\left[\sup_{t,t^\prime} (X_t-X_{t^\prime})\right]\,.
\end{align}
Therefore it suffices to show
\begin{align}
E\left[\sup_{t,t^\prime\in T}(X_t-X_{t^\prime})\right]\leq 2E\left[\sup_{d(t,\hat{t})<2\eta}(X_t-X_{\hat{t}})\right]+1_{\eta<D_T}8\sqrt{2}\int_{\eta/2}^{D_T/2}\sqrt{\log N(\epsilon,T,d)} d\epsilon
\end{align}
for all $\eta>0$.  If $\eta\geq D_T$ then  this is trivial, as all $t,\hat{t}\in T$ satisfy $d(t,\hat{t})< 2\eta$.  If $\eta<D_T$ and $N(\epsilon,T,d)=\infty$ for some $\epsilon>\eta/2$ then the right-hand side equals $\infty$ due to $\epsilon\mapsto N(\epsilon,T,d)$ being non-increasing, and hence the bound is again trivial. Therefore, from here on we restrict our attention to  the case where $\eta<D_T$  and   $N(\epsilon,T,d)<\infty$ for all $\epsilon>\eta/2$; in particular, this implies  $D_T$ is finite.

Let $k=\lfloor \log_2(D_T/\eta)\rfloor\in\mathbb{Z}_0$ so that $\eta\in(2^{-(k+1)}D_T, 2^{-k}D_T]$.  Therefore $N(2^{-i}D_T,T,d)<\infty$ for all $i=0,...,k$.  For $i=0,..,k$ let $\hat{T}_i$ be a minimal $2^{-i}D_T$-cover of $T$, i.e., $|\hat{T}_i|=N(2^{-i}D_T,T,d)<\infty$. Note that $|\hat{T}_0|=1$, as any element of $T$ is a minimal $D_T$-cover.   

For $i=1,...,k$, given $t\in \hat{T}_i$ there exists $\tilde{t}\in \hat{T}_{i-1}$ such that $d(\tilde{t},t)\leq 2^{-(i-1)}D_T$.  Therefore we can  define $g_i:\hat{T}_i\to \hat{T}_{i-1}$ such that $d(g_i(t),t)\leq 2^{-(i-1)}D_T$.  For $i=0,...,k$ inductively define $\pi_{k-i}:\hat{T}_k\to \hat{T}_{k-i}$ by  $\pi_k=Id$, $\pi_{k-i}=g_{k-(i-1)}\circ \pi_{k-(i-1)}$. Therefore for all $t\in\hat{T}_k$, $i=1,...,k$ we have  $d(\pi_{i-1}(t),\pi_i(t))=d(g_i(\pi_i(t)),\pi_i(t))\leq 2^{-(i-1)}D_T$.

Noting that $\pi_0$ is a constant map,  for $\hat{t},\hat{t}^\prime\in \hat{T}_k$ we can now compute
\begin{align}
X_{\hat{t}}-X_{\hat{t}^\prime}=&\sum_{i=1}^k (X_{\pi_i(\hat{t})}-X_{\pi_{i-1}(\hat{t})})+\sum_{i=1}^k (X_{\pi_{i-1}(\hat{t}^\prime)}-X_{\pi_{i}(\hat{t}^\prime)})\\
=&\sum_{i=1}^k (X_{\pi_i(\hat{t})}-X_{g_i(\pi_{i}(\hat{t}))})+\sum_{i=1}^k (X_{g_i(\pi_{i}(\hat{t}^\prime))}-X_{\pi_{i}(\hat{t}^\prime)})\notag\\
\leq&\sum_{i=1}^k \sup_{\hat{t}\in \hat{T}_i} (X_{\hat{t}}-X_{g_i(\hat{t})})+\sum_{i=1}^k \sup_{\hat{t}\in \hat{T}_i}  (X_{g_i(\hat{t})}-X_{\hat{t}})\,.\notag
\end{align}

For $t,t^\prime\in T$ there exists $\hat{t},\hat{t}^\prime\in\hat{T}_k$ with $d(t,\hat{t})\leq 2^{-k}D_T$ and $d(t^\prime,\hat{t}^\prime)\leq 2^{-k}D_T$. Hence
\begin{align}
&X_t-X_{t^\prime}= X_t-X_{\hat{t}}+X_{\hat{t}}-X_{\hat{t}^\prime}+X_{\hat{t}^\prime}-X_{t^\prime}\\
\leq& 2\sup_{d(t,\hat{t})\leq2^{-k}D_T}(X_t-X_{\hat{t}})+\sum_{i=1}^k \sup_{\hat{t}\in \hat{T}_i} (X_{\hat{t}}-X_{g_i(\hat{t})})+\sum_{i=1}^k \sup_{\hat{t}\in \hat{T}_i}  (X_{g_i(\hat{t})}-X_{\hat{t}})\,.\notag
\end{align}

Therefore we can use the sub-Gaussian case of the maximal inequality from  Theorem 2.5 in \cite{boucheron2013concentration} together with the fact that $\epsilon\mapsto N(\epsilon,T,d)$ is non-increasing to compute
\begin{align}
&E\left[\sup_{t,t^\prime\in T}(X_t-X_{t^\prime})\right]\\
\leq& 2E\left[\sup_{d(t,\hat{t})\leq2^{-k}D_T}(X_t-X_{\hat{t}})\right]+\sum_{i=1}^k E\left[\sup_{\hat{t}\in \hat{T}_i} (X_{\hat{t}}-X_{g_i(\hat{t})})\right]+\sum_{i=1}^k E\left[\sup_{\hat{t}\in \hat{T}_i}  (X_{g_i(\hat{t})}-X_{\hat{t}})\right]\notag\\
\leq&2E\left[\sup_{d(t,\hat{t})\leq2^{-k}D_T}(X_t-X_{\hat{t}})\right]+2\sum_{i=1}^k \max_{\hat{t}\in \hat{T}_i}d(\hat{t},g_i(\hat{t}))\sqrt{2\log|\hat{T}_i|}\notag\\
\leq&2E\left[\sup_{d(t,\hat{t})\leq2^{-k}D_T}(X_t-X_{\hat{t}})\right]+2\sum_{i=1}^k 2^{-(i-1)}D_T\sqrt{2\log N(2^{-i}D_T,T,d)}\notag\\
\leq&2E\left[\sup_{d(t,\hat{t})\leq2^{-k}D_T}(X_t-X_{\hat{t}})\right]+8\sqrt{2}\sum_{i=1}^k \int_{2^{-i-1}D_T}^{2^{-i}D_T}\sqrt{\log N(\epsilon,T,d)} d\epsilon\notag\\
\leq&2E\left[\sup_{d(t,\hat{t})< 2\eta}(X_t-X_{\hat{t}})\right]+8\sqrt{2}\int_{\eta/2}^{\frac{1}{2}D_T}\sqrt{\log N(\epsilon,T,d)} d\epsilon\,,\notag
\end{align}
where to obtain the last line we used  $2^{-k-1}D_T\geq \eta/2$ and $2\eta>2^{-k}D_T$.  This  completes the proof.
\end{proof}

We now use the above version of Dudley's entropy integral  to derive a distribution-dependent Rademacher complexity bound.  We work with collections of functions that have the following properties.

\begin{theorem}\label{thm:Rademacher_bound_unbounded_support}
Let $(\Theta,d_\Theta)$ be a non-empty pseudometric space and denote its diameter by $D_\Theta\coloneqq\sup_{\theta_1,\theta_2\in\Theta}d_\Theta(\theta_1,\theta_2)$.  Let $(\mathcal{Z},P)$ be a probability space and suppose we have a collection of measurable   functions $g_\theta:\mathcal{Z}\to \mathbb{R}$, $\theta\in\Theta$. Define $\mathcal{G}\coloneqq \{g_\theta:\theta\in\Theta\}$.

If for all $z\in\mathcal{Z}$  there exists $L(z)\in(0,\infty)$ such that $\theta\mapsto g_\theta(z)$ is $L(z)$-Lipschitz then
 for all $n\in\mathbb{Z}^+$ and all $z\in\mathcal{Z}^n$ we have the following bound on the empirical Rademacher complexity: 
\begin{align}\label{eq:emp_rademacher_bound_L}
\widehat{\mathcal{R}}_{\mathcal{G},n}(z)\leq&L_n(z)\inf_{\eta>0}\left\{4\eta+1_{\eta<  D_\Theta}8\sqrt{2}n^{-1/2}\int_{\eta/2}^{ D_\Theta/2}\sqrt{\log  N(\epsilon,\Theta,d_{\Theta})}  d\epsilon\right\}\,, 
\end{align} 
where $L_n(z)\coloneqq \left(\frac{1}{n}\sum_{i=1}^n L(z_i)^2\right)^{1/2}$. Assuming the requisite measurability properties for the expected values to be defined, we also obtain the $P$-Rademacher complexity bound
\begin{align}\label{eq:rademacher_bound_L}
\mathcal{R}_{\mathcal{G},P,n}\leq   E_P[L^2]^{1/2}\inf_{\eta>0}\left\{4\eta+1_{\eta<  D_\Theta}8\sqrt{2}n^{-1/2}\int_{\eta/2}^{ D_\Theta/2}\sqrt{\log  N(\epsilon,\Theta,d_{\Theta})}  d\epsilon\right\}\,.
\end{align}
 
Moreover, if $\Theta$ is the unit ball in $\mathbb{R}^k$ under some norm and $d_\Theta$ is the corresponding metric  then  we have
\begin{align}\label{eq:entropy_int_bound_unit_ball}
\inf_{\eta>0}\left\{4\eta+1_{\eta<  D_\Theta}8\sqrt{2}n^{-1/2}\int_{\eta/2}^{ D_\Theta/2}\sqrt{\log  N(\epsilon,\Theta,d_{\Theta})}  d\epsilon\right\}  \leq 32 (k/n)^{1/2}\,.
\end{align}
\end{theorem}
\begin{proof}
Let $\sigma_i$, $i=1,...,n$ be independent uniform random variables taking values in $\{-1,1\}$. It is well-known that the Rademacher process  $X_t\coloneqq \frac{t\cdot \sigma}{n}$, $t\in\mathbb{R}^n$, is a mean-zero sub-Gaussian process with respect to   the norm $\|\cdot\|_2/n$, see, e.g., Chapter 5 in \cite{wainwright2019high}.  Therefore, for $T\subset\mathbb{R}^n$,  Theorem \ref{thm:Dudley} implies
\begin{align} \label{eq:emp_rad_bound1}
&E_\sigma\left[\sup_{t\in T}  \frac{t\cdot \sigma}{n}\right]\\
\leq&\inf_{\eta>0}\left\{2E_\sigma\left[\sup_{\|t-\hat{t}\|_2/n\leq2\eta}\frac{(t-t^\prime)\cdot \sigma}{n}\right]+1_{\eta<D_T}8\sqrt{2}\int_{\eta/2}^{D_T/2}\sqrt{\log N(\epsilon,T,\|\cdot\|_2/n)} d\epsilon\right\}\notag\\
\leq&\inf_{\eta>0}\left\{4\eta\sqrt{n}+1_{\eta<D_T}8\sqrt{2}\int_{\eta/2}^{D_T/2}\sqrt{\log N(\epsilon,T,\|\cdot\|_2/n)} d\epsilon\right\}\,,\notag
\end{align}
where $D_T$ is the diameter of $T$ under $\|\cdot\|_2/n$.

Now define  the map 
\begin{align}
\theta\in\Theta\mapsto H(\theta)\coloneqq (g_\theta(z_1),...,g_\theta(z_n))\,.
\end{align}   Defining a norm  on $\mathbb{R}^n$  by $\|t\|_{L^2(n)}\coloneqq\sqrt{\frac{1}{n}\sum_{i=1}^n t_i^2}=\|t\|_2/\sqrt{n}$, we can   compute
\begin{align}
\|H(\theta_1)-H(\theta_2)\|_{L^2(n)}^2=&\frac{1}{n}\sum_{i=1}^n(g_{\theta_1}(z_i)-g_{\theta_2}(z_i))^2\\
\leq&\frac{1}{n}\sum_{i=1}^n L(z_i)^2d_\Theta(\theta_1,\theta_2)^2\,.\notag
\end{align}
Therefore $H$ is   $L_n(z)\coloneqq \left(\frac{1}{n}\sum_{i=1}^n L(z_i)^2\right)^{1/2}$-Lipschitz with respect to $(d_{\Theta},\|\cdot\|_{L^2(n)})$ and $H(\Theta)=   \mathcal{G}_n(z)\coloneqq \{  (g_\theta(z_1),...,g_\theta(z_n)):\theta\in\Theta\}$.    Therefore we can conclude the following relation between covering numbers,
\begin{align}
N(\epsilon,\mathcal{G}_n(z),\|\cdot\|_{L^2(n)})\leq N(\epsilon/L_n(z),\Theta,d_{\Theta})\,,
\end{align}
as well as the diameter bounds $D_{\mathcal{G}_n(z),L^2(n)}\leq L_n(z) D_\Theta$, where $D_{\mathcal{G}_n(z),L^2(n)}$ denotes the diameter of $\mathcal{G}_n(z)$ under the  $L^2(n)$ norm.

Letting $T=\mathcal{G}_n(z)$ in \eqref{eq:emp_rad_bound1}, using  $D_T
= D_{T, {L_2(n)}}/\sqrt{n}$ and
\begin{align}
N(\epsilon,T,\|\cdot\|_{2}/{n})=N(\epsilon,T,\|\cdot\|_{L^2(n)}/\sqrt{n})=N(\sqrt{n}\epsilon,T,\|\cdot\|_{L^2(n)})\,,
\end{align}
 changing variables in the integral and the infimum, and then using the above covering number and diameter bounds we can  compute 
\begin{align}
&\widehat{\mathcal{R}}_{\mathcal{G},n}(z)\\
\leq&\inf_{\eta>0}\left\{4\eta\sqrt{n}+1_{\sqrt{n}\eta<D_{\mathcal{G}_n(z),L^2(n)}}8\sqrt{2}\int_{\eta/2}^{D_{\mathcal{G}_n(z),L^2(n)}/(2\sqrt{n})}\sqrt{\log N(\epsilon,\mathcal{G}_n(z),\|\cdot\|_{L^2(n)}/\sqrt{n})} d\epsilon\right\}\notag\\
=&\inf_{\eta>0}\left\{4\eta+1_{\eta<D_{\mathcal{G}_n(z),L^2(n)}}8\sqrt{2}n^{-1/2}\int_{\eta/2}^{D_{\mathcal{G}_n(z),L^2(n)}/2}\sqrt{\log N(\epsilon,\mathcal{G}_n(z),\|\cdot\|_{L^2(n)})}  d\epsilon\right\}\notag\\
\leq&\inf_{\eta>0}\left\{4\eta+1_{\eta< L_n(z) D_\Theta}8\sqrt{2}n^{-1/2}\int_{\eta/2}^{ L_n(z) D_\Theta/2}\sqrt{\log  N(\epsilon/L_n(z),\Theta,d_{\Theta})}  d\epsilon\right\}\notag\\
=&L_n(z)\inf_{\eta>0}\left\{4\eta+1_{\eta<  D_\Theta}8\sqrt{2}n^{-1/2}\int_{\eta/2}^{ D_\Theta/2}\sqrt{\log  N(\epsilon,\Theta,d_{\Theta})}  d\epsilon\right\}\,.\notag
\end{align}
This proves \eqref{eq:emp_rademacher_bound_L}. Taking the expectation of this empirical Rademacher complexity bound with respect to $P^n$ and then using the Cauchy-Schwarz inequality to compute
\begin{align}
E_{P^n}[L_n]\leq\left(E_{P^n}\left[\frac{1}{n}\sum_{i=1}^n L(z_i)^2\right]\right)^{1/2}=E_P[L^2]^{1/2}
\end{align}
we arrive at \eqref{eq:rademacher_bound_L}.

Finally, if $\Theta$ is the unit ball in $\mathbb{R}^k$ with respect to the norm $\|\cdot\|_{\Theta}$ and $d_\Theta$ is the associated metric then $D_{\Theta}\leq 2$ and we have the covering number bound  $N(\epsilon,\Theta,d_{\Theta})\leq (1+2/\epsilon)^k$ (see, e.g., Example 5.8 in \cite{wainwright2019high}).  In this case the covering number term is integrable at $\epsilon=0$ and so we simplify the bound  by taking $\eta\to 0$ and compute
\begin{align}
&\inf_{\eta>0}\left\{4\eta+1_{\eta<  D_\Theta}8\sqrt{2}n^{-1/2}\int_{\eta/2}^{ D_\Theta/2}\sqrt{\log  N(\epsilon,\Theta,d_{\Theta})}  d\epsilon\right\}\\
\leq&8\sqrt{2}  (k/n)^{1/2}\int_{0}^{ 1}\sqrt{\log(1+2/\epsilon)}  d\epsilon\notag\\
\leq& 32  (k/n)^{1/2}\,.\notag
\end{align}
 This proves \eqref{eq:entropy_int_bound_unit_ball}.
\end{proof}

\section{Uniform Law of Large Numbers Mean Bound  for Unbounded Functions}\label{app:ULLN_rademacher}
In this appendix we  present a uniform law of large numbers (ULLN) mean bound that can  be paired with Theorem  \ref{thm:Rademacher_bound_unbounded_support} to derive computable distribution-dependent ULLN concentration inequalities for appropriate families of unbounded functions.    The argument 
follows the well-known technique  of using symmetrization to convert the problem to one of bounding the Rademacher complexity; see, e.g.,  Theorem 3.3 in \cite{mohri2018foundations}, especially the argument from Eq. (3.8)-(3.13).   The only distinction from the result in  \cite{mohri2018foundations}, which assumes the functions are valued in a bounded interval,  is  that we generalize to appropriate families of unbounded functions, as is required, e.g., for the application to DSM in Section \ref{sec:DSM} of the main text.  Despite this generalization, the proof is essentially the same; we present it here for completeness.

\begin{theorem}\label{thm:ULLN_unbounded}
Let $(\mathcal{Z},P)$ be a probability space and $\mathcal{G}$ be a countable family of real-valued measurable functions on $\mathcal{Z}$.  Suppose we have  $h\in L^1(P)$ such that $|g|\leq h$ for all $g\in\mathcal{G}$.  Then for $n\in\mathbb{Z}^+$ we have
\begin{align}
E_{z\sim P^n}\!\left[\sup_{g\in\mathcal{G}}\left\{\pm\left(\frac{1}{n}\sum_{i=1}^n g(z_i)-E_P[g]\right)\right\}\right]\leq 2\mathcal{R}_{\mathcal{G},P,n}\,.
\end{align}
\end{theorem}
\begin{remark}
Note that the argument of the supremum is invariant under shifts $g\to g-c_g$, $c_g\in\mathbb{R}$, therefore  one can also  apply this result  to   $\mathcal{G}$'s such that there exists shifts $c_g\in\mathbb{R}$ and $h\in L^1$ satisfying  $|g-c_g|\leq h$ for all $g\in\mathcal{G}$.   Finally, by the usual limiting argument, this result can be applied to appropriate uncountable $\mathcal{G}$ that satisfy a separability property.
\end{remark}
\begin{proof}
First consider the $+$ sign. We have $\mathcal{G}\subset  L^1(P)$, $\sup_{g\in\mathcal{G}}\{\frac{1}{n}\sum_{i=1}^n g(z_i)-E_P[g]\}$ is measurable, and 
\begin{align}
\left|\sup_{g\in\mathcal{G}}\left\{\frac{1}{n}\sum_{i=1}^n g(z_i)-E_P[g]\right\}\right|\leq  \frac{1}{n}\sum_{i=1}^n h(z_i)+E_P[h]\in L^1(P^n)\,.
\end{align}
Therefore the expectation on the left-hand side of the claimed bound  exists and is finite. Next  use a symmetrization technique to compute
\begin{align}
&E_{z\sim P^n}\!\left[\sup_{g\in\mathcal{G}}\left\{\frac{1}{n}\sum_{i=1}^n g(z_i)-E_P[g]\right\}\right]\\
=&\int \sup_{g\in\mathcal{G}}\left\{\int\frac{1}{n}\sum_{i=1}^n (g(z_i)-  g(z_i^\prime))P^n(dz^\prime)\right\} P^n(dz)\notag\\
\leq& \int \sup_{g\in\mathcal{G}}\left\{\frac{1}{n}\sum_{i=1}^n (g(z_i)- g(z_i^\prime))\right\}P^{2n}(dz^\prime,dz)\notag\,.
\end{align}
For convenience, we index $(z,z^\prime)\in\mathcal{Z}^{2n}$ by $(z_{-n},...,z_{-1},z_1,...,z_n)$. For $\sigma\in\{-1,1\}^n$ define $H_\sigma:\mathcal{Z}^{2n}\to\mathcal{Z}^{2n}$, $H_\sigma[z]_i=z_{\sigma_{|i|}\cdot i}$.  The pushforward satisfies $(H_\sigma)_{\#}P^{2n}=P^{2n}$ and therefore we can compute
\begin{align}
&E_{z\sim P^n}\!\left[\sup_{g\in\mathcal{G}}\left\{\frac{1}{n}\sum_{i=1}^n g(z_i)-E_P[g]\right\}\right]\\
\leq& \int \sup_{g\in\mathcal{G}}\left\{\frac{1}{n}\sum_{i=1}^n (g(H_\sigma[z]_i)- g(H_\sigma[z]_{-i}))\right\}P^{2n}(dz_{-n},...,dz_{-1},dz_1,...,dz_{n})\notag\\
=&\int \sup_{g\in\mathcal{G}}\left\{\frac{1}{n}\sum_{i=1}^n (g(z_{\sigma_i\cdot i})- g(z_{-\sigma_i\cdot i}))\right\}P^{2n}(dz_{-n},...,dz_{-1},dz_1,...,dz_{n})\notag\\
=&\int \sup_{g\in\mathcal{G}}\left\{\frac{1}{n}\sum_{i=1}^n \sigma_i(g(z_{i})- g(z_{- i}))\right\}P^{2n}(dz_{-n},...,dz_{-1},dz_1,...,dz_{n})\notag\\
\leq& \int \sup_{g\in\mathcal{G}}\left\{\frac{1}{n}\sum_{i=1}^n \sigma_ig(z_{i})\right\}P^{n}(dz_1,...,dz_{n})+\int  \sup_{g\in\mathcal{G}}\left\{\frac{1}{n}\sum_{i=1}^n (-\sigma_i) g(z_{i}))\right\}P^{n}(dz_1,...,dz_{n})\,.\notag
\end{align}
Taking the expectation with respect to the uniform distribution over $\sigma\in\{-1,1\}^n$ we therefore find
\begin{align}
&E_{z\sim P^n}\!\left[\sup_{g\in\mathcal{G}}\left\{\frac{1}{n}\sum_{i=1}^n g(z_i)-E_P[g]\right\}\right]
\leq 2\mathcal{R}_{\mathcal{G},P,n}\,.
\end{align}
 This proves the result with the $+$ sign. To prove it with the $-$ sign, simply note that essentially the same argument goes through with  $\mathcal{G}$ replaced by $-\mathcal{G}$.
\end{proof}
\section{Mean Bound and ULLN in Finite Dimensions}\label{app:mean_ULLN_finite_dim}
For convenience of the reader, in this appendix we provide a streamlined and self-contained result that combines  Lemma \ref{lemma:variable_reuse_expectation_bound} and Theorem \ref{thm:ULLN_var_reuse_mean_bound}  and which is appropriate for finite-dimensional applications.
\begin{corollary}
Let $\|\cdot\|$ be a norm on $\mathbb{R}^k$, $\Theta\subset\mathbb{R}^k$ be nonempty,  $(\mathcal{X},P_X)$, $(\mathcal{Y},P_Y)$ be probability spaces,   and $g_\theta:\mathcal{X}\times\mathcal{Y}\to\mathbb{R}$, $\theta\in\Theta$, be measurable  such that the following hold:
\begin{enumerate}
\item  We have  $c:\Theta\to\mathbb{R}$ and   $h:\mathcal{X}\times\mathcal{Y}\to[0,\infty)$  such that $h\in L^1(P_X\times P_Y)$ and $|g_\theta-c_\theta|\leq h$ for all $\theta\in\Theta$.
\item We have a measurable $L:\mathcal{X}\times\mathcal{Y}\to (0,\infty)$ such that  $\theta\mapsto g_\theta(x,y)-c_\theta$ is $L(x,y)$-Lipschitz  with respect to $\|\cdot\|$ for all $x\in\mathcal{X}$, $y\in\mathcal{Y}$.
\item We have  $h_{\mathcal{X}}\in L^1(P_X\times P_X\times P_Y)$ such that $|g_\theta(x,y)-g_\theta(\tilde{x},y)|\leq h_{\mathcal{X}}(x,\tilde{x},y)<\infty$ for all $x,\tilde x\in \mathcal{X}$, $y\in \mathcal{Y}$, $\theta\in\Theta$.
\item We have $h_{\mathcal{Y}}\in L^1(P_Y\times P_Y)$ such that $|g_\theta(x,y)-g_\theta(x,\tilde{y})|\leq h_{\mathcal{Y}}(y,\tilde{y})<\infty$ for all $x\in \mathcal{X}$, $y,\tilde{y}\in \mathcal{Y}$, $\theta\in\Theta$.
\item There exists $K_X\in(0,\infty]$, $\xi_X:[0,K_X)\to[0,\infty)$ such that
\begin{align}
\int  \cosh\left(\lambda  E_{P_Y}[h_{\mathcal{X}}(x,\tilde{x},\cdot)]\right)  (P_X\times P_X)(dxd\tilde{x}) \leq e^{\xi_X(|\lambda|)}\,\,\,\,\text{ for all $\lambda\in(-K_X,K_X)$.}
\end{align}
\item There exists $K_Y\in(0,\infty]$, $\xi_Y:[0,K_Y)\to[0,\infty)$ such that
\begin{align}
\int \cosh\left(\lambda h_{\mathcal{Y}}(y,\tilde{y})\right)  (P_Y\times P_Y)(dyd\tilde{y})\leq e^{\xi_Y(|\lambda|)}\,\,\,\,\text{ for all $\lambda\in(-K_Y,K_Y)$.}
\end{align}
 \item  We have  $m\in\mathbb{Z}^+$ such that $K_X/m\leq K_{Y}$ .
\end{enumerate}

For  $n\in\mathbb{Z}^+$ define $P^{n,m}\coloneqq \prod_{i=1}^n \left(P_X\times \prod_{j=1}^m P_Y\right)$ and  $\phi: \prod_{i=1}^n\left(\mathcal{X}\times \mathcal{Y}^m\right)\to  \mathbb{R}$ by
\begin{align}
&\phi_\pm(z_1,...,z_n)\coloneqq\sup_{\theta\in\Theta}\left\{\pm\left(\frac{1}{nm}\sum_{i=1}^n\sum_{j=1}^mg_\theta(x_i,y_{i,j})-E_{P_X\times P_Y}[g_\theta]\right)\right\}\,,
\end{align}
where $z_i\coloneqq (x_i,y_{i,1},...,y_{i,m})$. Then   $\phi_\pm \in L^1({P^{n,m}})$,
\begin{align}
&E_{P^{n,m}}\!\left[\phi_\pm\right]\leq \widetilde{C}_{n,m}\,,
\end{align}
 and  for all $t\geq 0$ we have
\begin{align}
P^{n,m}\left(\phi_\pm\geq t+\widetilde{C}_{n,m}\right)\leq \exp\left(-n\sup_{\lambda\in[0,K_X)}\left\{\lambda t-\left( \xi_X(\lambda)+m \xi_Y(\lambda/m)\right)\right\}\right)\,,
\end{align}
where 
\begin{align} 
\widetilde{C}_{n,m}\coloneqq& \frac{16\sqrt{2}}{n^{1/2}}\left((1-1/m)\int  E_{P_Y(dy)}[ L(x,y)] ^2 P_X(dx)+\frac{1}{m} E_{P_X\times P_Y}[L(x,y)^2]\right)^{1/2}\\
&\times \int_{0}^{D_\Theta/2}\sqrt{\log N(\epsilon,\Theta,\|\cdot\|)}d\epsilon\,,\notag\\
D_\Theta\coloneqq&\sup_{\theta_1,\theta_2\in\Theta}\|\theta_1-\theta_2\|\notag\,.
\end{align}

\end{corollary}

\bibliographystyle{siamplain}
\bibliography{Concentration_ineq_stochastic_opt_arxiv.bbl}

\end{document}